\documentclass[11pt]{article}
\usepackage{geometry}
\usepackage{graphicx}
\usepackage{subcaption}
\usepackage{xcolor}
\usepackage{natbib}
\usepackage{amsmath}
\usepackage{amssymb}
\usepackage{amsthm}
\usepackage{bbm}
\usepackage{algorithm}
\usepackage{algorithmicx}
\usepackage{float}
\usepackage{enumitem}
\usepackage{indentfirst}
\usepackage[colorlinks,
            linkcolor=blue,
            anchorcolor=red ,
            citecolor=blue
            ]{hyperref}

\usepackage{hyperref}
\usepackage{cleveref}

\newtheorem{thm}{Theorem}%[section]
\newtheorem{lem}{Lemma}%[section]
\newtheorem{pro}{Proposition}%[section]
\newtheorem{cor}[thm]{Corollary}%[section]
\newtheorem{defi}{Definition}[section]

\newtheorem{ass}{Assumption}
\newcommand{\Xs}{\mathcal{X}}

\newcommand{\As}{\mathcal{A}}
\newcommand{\Ds}{\mathcal{D}}
\newcommand{\Ss}{\mathcal{S}}

\newcommand{\Eb}{\mathbb{E}}
\newcommand{\Pb}{\mathbb{P}}
\newcommand{\Rb}{\mathbb{R}}
\newcommand{\drm}{\mathrm{d}}
\newcommand{\Normal}{\mathcal{N}}  % distributions
\newcommand{\Unif}{\mathrm{U}}
\newcommand{\Bin}{\mathrm{Bin}}
\newcommand{\coloneq}{\mathrel{\mathop:}=}

\title{PAC Off-Policy Prediction of Contextual Bandits}
\author{Yilong Wan, Yuqiang Li, Xianyi Wu}
\date{\today}

\begin{document}

\maketitle

\begin{abstract}
\noindent This paper investigates off-policy evaluation in contextual bandits, aiming to quantify the performance of a target policy using data collected under a different and potentially unknown behavior policy. Recently, methods based on conformal prediction have been developed to construct reliable prediction intervals that guarantee marginal coverage in finite samples, making them particularly suited for safety-critical applications. To further achieve coverage conditional on a given offline data set, we propose a novel algorithm that constructs probably approximately correct prediction intervals. Our method builds upon a PAC-valid conformal prediction framework, and we strengthen its theoretical guarantees by establishing PAC-type bounds on coverage. We analyze both finite-sample and asymptotic properties of the proposed method, and compare its empirical performance with existing methods in simulations.

\end{abstract}

%%%%%%%%%%%%%%%%%%%%%%%%%%%%%%%%%%%%%%%%%%%%%%%%%%%%%%%%%%%%%%%%%%%%%%%%%%%%%%%%%%%%%%%%
%%%%%%%%%%%%%%%%%%%%%%%%%%%%%%%%%%%%%%%%%%%%%%%%%%%%%%%%%%%%%%%%%%%%%%%%%%%%%%%%%%%%%%%%

\section{Introduction}\label{sec1}

In many fields such as healthcare, marketing, content recommendation,  and aotonomous driving, it is often impractical to directly test and improve a policy in the real world due to ethical considerations, resource constraints and associated risks. In addition, it is sometimes important to understand the potential impact of a decision-making policy prior to deployment.  Therefore, one may seek to evaluate policies using offline data previously collected under usually different behavior policies. This process is known as off-line evaluation, or more generally, off-policy evaluation (OPE).

Contextual bandits are building block models in reinforcement learning (RL). Under a contextual bandit setting, at each time step, the agent observes a context, selects an action according to a given policy, then receives a random reward from the environment that depends on the context-action pair, and after that the environment transition to a new context. Thus, contextual bandits lie between the most general Markov decision process models (requiring general RL algorithms) and the simplest multi-armed bandit processes (requiring more special MAB algorithms) and also provide powerful tools to certain decision-making problems. In a mathematical view of point, it bridges the gap between supervised learning and reinforcement learning.

While the majority of prior work has focused on the expected rewards, they may not be suitable in safety-critical situations due to their inability to capture the variability of the reward. Consequently, recent literature has witnessed the research progress on alternative measures of performance, including variance, quantiles, and conditional values at risk, among others; see e.g., \cite{keramati2020being}, \cite{chandak2021universal} and \cite{huang2021off}.

A promising approach for uncertainty quantification is through prediction interval (PI), which contains the true reward itself with a specified probability. For contextual bandits, \cite{taufiq2022conformal} proposed an algorithm to generate finite-sample valid PIs with stochastic policies and continuous action spaces.  In their method,  PIs are adaptive to test contexts (see Section \ref{sec2} for a definition) and thus of significant interest in fields such as precision medicine \citep{lei2021conformal}. However, their approach requires estimating the probability densities of rewards conditional on context-action pairs, which can be challenging when the model is unknown. \cite{zhang2023conformal} addressed this limitation by  introducing a sub-sampling-based method and extended the framework to both contextual bandits and sequential decision-making scenarios in the environments with discrete action spaces.

Both \cite{taufiq2022conformal} and \cite{zhang2023conformal} employ conformal prediction (CP) \citep{vovk2005algorithmic, shafer2008tutorial, balasubramanian2014conformal}, a well-established and effective method for uncertainty quantification. Through a post-training calibration step, CP guarantees a user-specified coverage in finite samples, relying solely on the exchangeability between the calibration data set and the test point, irrespective of the underlying model or data distribution. This makes CP particularly appealing for safe deployment in high-stakes applications. However, a key limitation of CP is that its validity is inherently unconditional (or marginal), as the coverage holds only under the randomness of both the calibration and test data. This marginal validity can be less appealing when coverage conditional on a fixed calibration data set or individual test point is desired, in which case the resulting PIs can systematically undercover without proper control.

We aim to establish probably approximately correct (PAC) PIs as in \cite{vovk2013conditional}, ensuring that the desired coverage, conditional on the offline data, is achieved with a prescribed confidence level. Such PAC validity is appealing in risk-sensitive applications, where both coverage and a high confidence guarantee are required. To this end, we propose Probably Approximately Correct Off-Policy Prediction (PACOPP), an algorithm for constructing PAC prediction intervals for the OPE task in contextual bandits. PACOPP enjoys both finite-sample theoretical guarantees and adaptivity with respect to the test context, without relying on any distributional or space assumptions.

Our contributions are as follows:

(i) Methodologically, we develop a novel procedure to construct off-policy PAC prediction intervals for a target policy's reward at any test context in bandits. our method achieves stronger validity and is more general than previous works, as it does not require model estimation and can be applied to continuous actions.

(ii)% Theoretically, we provide finite-sample guarantees for both the validity and efficiency of PACOPP. In particular, we extend the theoretical results of the modified CP method with PAC validity \citep{park2020pac} by establishing, for the first time, PAC-type bounds in finite samples, thereby justifying its efficiency. Furthermore, we show that our PIs are asymptotically equivalent to the best ``oracle'' PI when the regression estimators are consistent. 
Theoretically, we establish finite-sample guarantees for both the validity and efficiency of PACOPP. The resulting prediction intervals are shown to be asymptotically equivalent to the optimal ``oracle'' PI when the underlying regression estimators are consistent. Notably, by deriving PAC-type bounds for the first time, we provide a rigorous justification for the efficiency of PAC-valid CP, thereby complementing the theoretical foundations of conformal prediction \citep{angelopoulos2025theoretical}.

The remainder of the paper is structured as follows. The problem is mathematically formulated in Section \ref{sec2} and then the framework of PACOPP is presented in Section  \ref{sec3} with  proofs and certain details postponed in the Appendix. Section \ref{sec4} reports some numerical studies to demonstrate the validity and efficiency of the proposed method. Some relevant literature is reviewed in Section \ref{sec5} and Section \ref{sec6} briefly summarizes the findings and outlines several potential future directions. 

%%%%%%%%%%%%%%%%%%%%%%%%%%%%%%%%%%%%%%%%%%%%%%%%%%%%%%%%%%%%%%%%%%%%%%%%%%%%%%%%%%%%%%%%
%%%%%%%%%%%%%%%%%%%%%%%%%%%%%%%%%%%%%%%%%%%%%%%%%%%%%%%%%%%%%%%%%%%%%%%%%%%%%%%%%%%%%%%%

\section{Problem Formulation}\label{sec2}

Roughly write $\Delta(\Xs)$  for the set of all probability distributions over a space $\Xs$ and $[n]$  the  set $\{1,2,\ldots,n\}$ for an integer $n>0$.
Use $S,A,$ and $R$ to represent context, action and reward of a contextual bandit,  taking values in spaces $\Ss,\As,$ and $\Rb$, respectively. %These spaces can be either discrete or continuous.
A policy $\pi$ is a mapping $\pi: \Ss\mapsto\Delta(\As)$. In the OPE setting, we assume access to i.i.d. observations $\Ds=\{S_i, A_i, R_i\}_{i=1}^n$, collected with a behavior policy $\pi_b$, such that $S_i\sim P_S\in\Delta(\Ss)$,  $A_i\sim\pi_b(\,\cdot\mid S_i)$ given $S_i$, and  $R_i\sim P_R(\,\cdot\mid S_i, A_i)$ given $(S_i, A_i$), with $P_R:\Ss\times\As\mapsto\Delta(\Rb)$ mapping a context-action pair to a distribution over $\Rb$.  OPE aims to quantify the target reward $R_{n+1}$ after observing a test context $S_{n+1}$ for a ``future'' contex-action pair $(S_{n+1}, A_{n+1}, R_{n+1})$ that is independent of $\Ds$ but under the under the target policy $\pi_e$ such that
\begin{equation}\label{JointDisTargt}
(S_{n+1},  R_{n+1})\sim P^{\pi_e}(s, r)\coloneq P_S( s) \int_\As P_R(r|s,a)\pi_e( a |s)\drm a.
\end{equation}

With the data set $\Ds$,  {\color{blue}CP-based methods} construct a  prediction interval (PI) $C_\Ds(S_{n+1})$ such that
\begin{equation}\label{eq-COPP}
	\Pb\left[R_{n+1}\in C_\Ds(S_{n+1})\right] \geq 1-\epsilon,
\end{equation}
where $\epsilon$ is a pre-specified failure rate. The probability in (\ref{eq-COPP}) is computed under the joint distribution of $(S_{n+1},R_{n+1})$ in \eqref{JointDisTargt} that is induced by $\pi_e$ but the data set $\Ds$ induced by $\pi_b$. Hence, \eqref{eq-COPP} unnecessarily induces the same conditional guarantee for given $S_{n+1}$ (referred to as object conditional validity) or data set $\mathcal{D}$ (referred to as training conditional validity). Due to the inherent impossibility of designing non-trivial algorithms that output distribution-free PIs with object conditional validity (see e.g., \citealt{foygel2021limits}), the extensive efforts have been put on probably approximately correct (PAC) PIs, an relaxed version of training conditional validity, as defined below.

\begin{defi}[\citealt{park2020pac}] \label{PAC}For a pair of $\epsilon, \delta\in(0,1)$, a set-valued function $\hat{C}: \Ss\mapsto 2^{\Rb}$ constructed from $\Ds$ is  {\it probably approximately correct} (PAC) if
	\[
		\Pb\left[\Pb\left(\left. R_{n+1}\in \hat C (S_{n+1})\right|\Ds\right)\geq 1-\epsilon \right] \geq 1-\delta.
	\] 
In this case, the PI $\hat C (S_{n+1})$ is referred to as an $(\epsilon,\delta)$-PAC PI of $R_{n+1}$.  
\end{defi}
%For any fixed realization of the data $\cal D$, one generally encounter two situations: the coverage probability of  $\hat C (S_{n+1})$ to $R_{n+1}$ is at least $1-\epsilon$ (success) or smaller than $1-\epsilon$ (failure). However, a PAC PI guarantees that there are at least $100(1-\delta)$ percent of chances in which one successes. 
To simplify the representations, we sometimes use the notation
$$L_{P^{\pi_e}}(\hat C)\coloneq \Pb\left[ \left. R_{n+1}\notin \hat C(S_{n+1})\right|\Ds \right],$$ so that a PAC PI can be defined by $\Pb\left[L_{P^{\pi_e}}(\hat C)\leq \epsilon \right] \geq 1 - \delta$.% that are probably approximately correct,

This paper is dedicated to develop algorithms that output PIs with PAC validity defined in Definition \ref{PAC} in a contextual bandits off-policy setting.

%%%%%%%%%%%%%%%%%%%%%%%%%%%%%%%%%%%%%%%%%%%%%%%%%%%%%%%%%%%%%%%%%%%%%%%%%%%%%%%%%%%%%%%%
%%%%%%%%%%%%%%%%%%%%%%%%%%%%%%%%%%%%%%%%%%%%%%%%%%%%%%%%%%%%%%%%%%%%%%%%%%%%%%%%%%%%%%%%

\section{PACOPP Algorithms}\label{sec3}

The key working condition underpinning standard CP methods is the exchangeability (including i.i.d.) between the test and observational data. However,  the difference between $\pi_e$ and $\pi_b$ induces a distribution shift of  $P^{\pi_e}$ of the test data from $P^{\pi_e}$ of the observational data in $\Ds$, making the standard CP methods inapplicable. An extension is  the weighted CP framework  \citep{tibshirani2019conformal} that,  when applied in OPE setting, requires the weighted exchangeability (including independence)  that necessitates estimating a density ratio \begin{equation}\label{COPP-weight}
	\frac{P^{\pi_e}}{P^{\pi_b}}( s,  r) = \frac{\int_\As P_R( r | s, a) \pi_e( a | s) \drm a }{\int_\As P_R( r | s, a) \pi_b( a|s) \drm a },
\end{equation}
dut to the unknown reward distribution $P_R$.

{\color{blue}We take another workable idea that picks out a subset from the data set $\Ds$  by  a rejection sampling (RS) strategy whose distribution resembles the target distribution \citep{neumann1951various, owen2013monte, park2022pac}.}

This section thus proceeds in three subsections: The RS procedure to estimate $P^{\pi_e}$ from the data collected by $\pi_b$, the PACOPP algorithm with known $\pi_b$ and the PACOPP algorithm with unknown $\pi_b$.

\subsection{RS Procedure}
Define a weight function 
\(\label{weightdef}
	w(s,a)\coloneq\frac{\pi_e}{\pi_b}( a|s),\hbox{ with the convention }w(s,a)=0\hbox{ if }\pi_b( a|s)=\pi_e(a|s)=0.
\)  Denote by $B=\inf\{b: \pi_e(a|s)\leq b\pi_b(a|s) \hbox{ for all }(s,a)\}$. Then $B\geq 1$ and the equality holds only when $\pi_e$ and $\pi_b$ are identical, because no distribution can be dominated strictly by another one. The paper proceeds silently under the assumption that $B<\infty$, which indicating $\displaystyle\sup_{(s,a)\in\Ss\times\As} w(s,a) <\infty$, a measure of the distribution dismatch between the target and behavior policies.
%\begin{ass}\label{ass-1}
%	The behavior and target policies satisfy the condition $\displaystyle\sup_{(s,a)\in\Ss\times\As} w(s,a) <\infty$. %, where $w(s,a)=0$ if $\pi_b( a|s)=\pi_e(a|s)=0$.
%\end{ass} 

Let $V_i, i=1,2, \dots\overset{iid}{\sim}U([0,1])$ be another sample independent of $\cal D$. The RS procedure first picks out a subset
\begin{equation}\label{def-Rs}
	\Ds_{\mathrm{rs}} \coloneq \left\{(S_i,R_i) \mid (S_i,A_i,R_i)\in\Ds , V_i\leq \frac{1}{B}w(S_i,A_i) \right\}
\end{equation}
of a random size $N_{\mathrm{rs}}\leq n$, and then sort the elements of $\Ds_{\mathrm{rs}} $ in ascending order by their original indices in $\Ds$ and denote by $Z_j$ the $j$-th so that $\Ds_{\mathrm{rs}}=\{Z_j, j\in[N_{\mathrm{rs}}]\}$. Then we have the following fundamental Proposition \ref{pro-1} which is a variant of Theorem 4.2 in \cite{owen2013monte} (proof in Appendix \ref{pf-pro-1}).

\begin{pro}\label{pro-1}
	Given $N_{\mathrm{rs}}$, the data $\{Z_j, j\in[N_{\mathrm{rs}}]\}\overset{\text{i.i.d.}}{\sim}P^{\pi_e}$, writing $\Ds_{\mathrm{rs}}\sim (P^{\pi_e})^{N_{\mathrm{rs}}}$.
\end{pro}

\subsection{PACOPP with Known $\pi_b$}
Inspired by \cite{vovk2013conditional} and \cite{park2020pac}, Proposition \ref{pro-1} above suggests an algorithm (Algorithm~\ref{al-known}) to construct a PAC PI of the form  in Definition \ref{PAC} when the behavior policy $\pi_b$ is known.

{\bf Algorithm.} Building on the data $\Ds_{\mathrm{rs}}$ from the RS procedure, we adapt the conformal quantile regression approach \citep{romano2019conformalized} with a modified cut-off threshold. The algorithm is outlined as three steps as below with a pseudo code presented in Algorithm \ref{al-known}.  

{\bf (1) Split data.} With a fixed proportion $\gamma\in(0,1)$, split the data set $\Ds_{\mathrm{rs}}$ into a training set $\Ds^{\mathrm{rs}}_1$ of size $L=N_{\mathrm{rs}}-M$ and a calibration set $\Ds^{\mathrm{rs}}_2$ of size $M = \lceil\gamma N_{\mathrm{rs}}\rceil$. 

{\bf (2) Construct candidates.} Fix a pair $(\epsilon_{\mathrm{lo}} , \epsilon_{\mathrm{up}} )$ with $\epsilon_{\mathrm{up}} - \epsilon_{\mathrm{lo}} = 1 - \epsilon$. Using the training set $\Ds^{\mathrm{rs}}_1$, compute the $\epsilon_{\mathrm{lo}}$ and $\epsilon_{\mathrm{up}}$ conditional quantiles $\hat{q}_{\epsilon_{\mathrm{lo}}}(S)$ and $\hat{q}_{\epsilon_{\mathrm{up}}}(S)$ of $R$ given $S$ by means of any suitable method, such as linear regression \citep{koenker1978regression}, neural networks \citep{taylor2000quantile} or random forests \citep{meinshausen2006quantile}, i.e., the algorithm $\mathcal{A}_\mathrm{qu}$ stated in the Input part of Algorithm \ref{al-known}. \footnote{For the sake of simplicity in the discussion, we assume no additional exogenous randomness is introduced by any of the quantile prediction algorithm.} Then, construct a family of intervals
\begin{equation}\label{def-Ctau}
	\hat{C}_\tau(S)    \coloneq[\hat{q}_{\epsilon_{\mathrm{lo}}}(S)-\tau,\hat{q}_{\epsilon_{\mathrm{up}}}(S)+\tau],
\end{equation} increasing in $\tau$.

{\bf (3) Pick up the desired.} For every $(S_i,R_i)\in\Ds^{\mathrm{rs}}_2$, define a non-conformity score
\begin{equation}\label{def-tilde-tau-i}
	\tau_i:=\min\{\tau: R_i\in [\hat{q}_{\epsilon_{\mathrm{lo}}}(S_i)-\tau,\hat{q}_{\epsilon_{\mathrm{up}}}(S_i)+\tau]\} =\max\{ \hat{q}_{\epsilon_{\mathrm{lo}}}(S_i)-R_i,R_i-\hat{q}_{\epsilon_{\mathrm{up}}}(S_i)\},
\end{equation}
so as to quantify the extent to which the observation $(S_i,R_i)\in\Ds^{\mathrm{rs}}_2$ ``conforms'' to the training set $\Ds^{\mathrm{rs}}_1$. Write $\tau_{(i)}, i=1,2,\dots, M$ for the increasingly ordered $\tau_i$'s and $\tau_{(M+1)}=\infty$, with ties arbitrarily broken. Of a binomial distribution $\Bin(M,\epsilon)$, denote by $F_{\Bin(M,\epsilon)}(\cdot)$ the cumulative distribution function and 
\begin{equation}\label{def-k}
	k(M,\epsilon,\delta)\coloneq\underset{k\in\{-1,0,\ldots,M-1\}}{\max}\{k: F_{\Bin(M,\epsilon)}(k)\leq\delta\}.
\end{equation}
the $\delta$-quantile, where $k(M,\epsilon,\delta)=-1$ if $F_{\Bin(M,\epsilon)}(0)=(1-\epsilon)^M>\delta$. Define \begin{equation}\label{def-tilde_tau}
	\tilde{\tau}\coloneq\tau_{(M-k(M,\epsilon,\delta))}.
\end{equation}  
We have the following Theorem \ref{thm-1} (proof in Appendix \ref{pf-thm-1}).

\begin{algorithm}[ht]
	\caption{PAC Off-Policy Prediction with known behavior Policy} 
	\label{al-known}        
	\begin{algorithmic}[1]     
		\Statex \textbf{Input:} Observational data $\Ds=\{S_i, A_i, R_i\}_{i=1}^n$; test context $S_{n+1}$; PAC parameters $(\epsilon,\delta)\in(0,1)^2$; behavior policy $\pi_b$; target policy $\pi_e$; quantile estimation algorithm $\mathcal{A}_\mathrm{qu}$; quantile levels $\epsilon_{\mathrm{lo}},\epsilon_{\mathrm{up}}$ with $\epsilon_{\mathrm{up}}-\epsilon_{\mathrm{lo}}=1-\epsilon$; calibration ratio $\gamma$
		\Statex \textbf{Output:} Prediction interval $\hat{C}_{\tilde{\tau}}(S_{n+1})$
		
		\vspace{2ex}
		\State \textbf{The procedure}:
		\State \quad $\Ds_{\mathrm{rs}} \gets \textsc{Rs} (\Ds, w(s,a) = \frac{\pi_e}{\pi_b}( a|s))$
		\State \quad Split $\Ds_{\mathrm{rs}}$ into $\Ds^{\mathrm{rs}}_1\cup\Ds^{\mathrm{rs}}_2$ with $|\Ds^{\mathrm{rs}}_2| =\lceil\gamma |\Ds_{\mathrm{rs}}| \rceil $ 
		\State \quad Train $\hat{q}_{\epsilon_{\mathrm{lo}}}, \hat{q}_{\epsilon_{\mathrm{up}}}$ using $\mathcal{A}_\mathrm{qu}$ on $\Ds^{\mathrm{rs}}_1$
		\State \quad \textbf{return} $\hat{C}_{\tilde{\tau}}(S_{n+1})$ = \textsc{Pac-Cp}($\Ds^{\mathrm{rs}}_2, S_{n+1},  \epsilon, \delta, \hat{q}_{\epsilon_{\mathrm{lo}}}, \hat{q}_{\epsilon_{\mathrm{up}}}$)
		
		\vspace{2ex}
		\State \textbf{function} \textsc{Rs}($\Ds$, $w(s,a)$):
		\State \quad Compute $B \gets \sup_{(s,a)\in\Ss\times\As}w(s,a)$
		\State \quad Sample $V_i\sim \Unif([0,1])$ for $i\in[|\Ds|]$
		\State \quad \textbf{return} $\left\{(S_i,R_i) \mid (S_i,A_i,R_i)\in\Ds , V_i\leq \frac{1}{B}w(S_i,A_i) \right\}$
		
		\vspace{2ex}
		\State \textbf{function} \textsc{Pac-Cp}($\Ds, S, \epsilon, \delta, \hat{q}_{\epsilon_{\mathrm{lo}}}, \hat{q}_{\epsilon_{\mathrm{up}}}$):
		\State \quad $k \gets \max\{k \in \{-1, \ldots, |\mathcal{D}| - 1\} : F_{\Bin(|\mathcal{D}|, \epsilon)}(k) \leq \delta\}$
		\State \quad $\tau_i \gets \max\{ \hat{q}_{\epsilon_{\mathrm{lo}}}(S_i)-R_i,R_i-\hat{q}_{\epsilon_{\mathrm{up}}}(S_i)\}$ for $(S_i,R_i)\in \Ds$
		\State \quad $\tilde{\tau} \gets (|\Ds|-k)$-th smallest value among $\{\tau_i\}$
		\State \quad \textbf{return} $[\hat{q}_{\epsilon_{\mathrm{lo}}}(S)-\tilde{\tau},\hat{q}_{\epsilon_{\mathrm{up}}}(S)+\tilde{\tau}]$
	\end{algorithmic}
\end{algorithm}	

\begin{thm}\label{thm-1}
	The output $\hat{C}_{\tilde{\tau}}(S_{n+1})$ of Algorithm~\ref{al-known} is a ($\epsilon,\delta$)-PAC PI of $R_{n+1}$, i.e., $\Pb\left[ L_{P^{\pi_e}}(\hat{C}_{\tilde{\tau}})\leq\epsilon\right] \geq 1-\delta$. 
\end{thm}

The procedure \textsc{Pac-Cp} in Algorithm~\ref{al-known} is different from standard split CP that uses empirical $1-\epsilon$ quantile of non-conformity scores to construct a $1-\epsilon$ marginally valid PI, which has been shown to simultaneously achieve $(\varepsilon,\delta)$-PAC validity for $\varepsilon \geq \epsilon + \sqrt{\frac{-\log \delta}{2M}}$ \citep{vovk2013conditional}. Clearly, in order to obtain an $(\epsilon, \delta)$-PAC valid PI, split CP needs to use the $1-\epsilon + \sqrt{\frac{-\log \delta}{2M}}$ quantile that requires  $M\geq(-\log \delta)/2\epsilon^2$ to ensure $1-\epsilon + \sqrt{\frac{-\log \delta}{2M}}\leq 1$. This additional requirement makes split CP inapplicable in certain circumstances. 

Recall that each $\tau_i$ corresponds to the smallest $\tau$ such that $R_i\in\hat{C}_\tau(S_i)$, see (\ref{def-tilde-tau-i}). It can be verified that the threshold $\tilde{\tau} = \tau_{(M-k(M,\epsilon,\delta))}$ can also be expressed as 
$$
\tilde{\tau} =\min \left\{\tau: \sum_{(S_i,R_i)\in\Ds^{\mathrm{rs}}_2} \mathbbm{1} [R_i\notin C_\tau(S_i)] \leq k(M,\epsilon,\delta)\right\},
$$
a formula originally proposed by \cite{park2020pac}. Intuitively, this procedure identifies the narrowest PI that satisfies an upper bound on the empirical miscoverage, a constraint specifically designed to ensure the PAC guarantee. Theorems \ref{lem-upperbound} and \ref{thm-bothsides} (proofs in Appendix \ref{pf-thm-bothsides} and \ref{pf-thm-bothsides-2}) further provide a theoretical justification for the efficiency of $\hat{C}_{\tilde{\tau}}$ by establishing PAC-type bounds on its miscoverage.

\begin{thm}\label{lem-upperbound}
	Suppose that there are no ties among the $\tau_i$'s almost surely. % For  $\tilde{\tau}$ defined in (\ref{def-tilde_tau}), we have
	Then, the output  $\hat{C}_{\tilde\tau}$ of Algorithm~\ref{al-known} satisfies for all $n\geq1$,
	\begin{equation}\label{eq-lem-2}
		\Pb\left[ L_{P^{\pi_e}}(\hat{C}_{\tilde{\tau}})\leq\epsilon\right] < 1-\delta+\frac{C}{\sqrt{n}}, %+e^{4m_0/B\gamma}e^{-2n/B^2},
	\end{equation}
	where $C=\frac{7B}{\sqrt{\gamma\epsilon(1-\epsilon)}}+\sqrt{ \lfloor m_0/\gamma \rfloor B} + \frac{B}{2}$, $B$ is the constant used in the RS procedure and $m_0=\log_{1-\epsilon}{\delta}$.
\end{thm}

Theorems \ref{thm-1} and \ref{lem-upperbound} together imply that $\hat{C}_{\tilde{\tau}}$ attains the nominal coverage with confidence converging to $1 - \delta$, indicating its essential dependence on $\delta$. To further characterize its tightness with respect to $\epsilon$, we derive a lower bound on the probability that the coverage does not substantially exceed the nominal level.

\begin{thm}\label{thm-bothsides}
	Assume the conditions in Theorem \ref{lem-upperbound}. 
	For any $\Delta_\epsilon\in (0,\epsilon)$ and $n\geq 1$, the output  $\hat{C}_{\tilde\tau}$ of Algorithm~\ref{al-known} satisfies
	\begin{equation}\label{eq-thm-bothsides}
		\Pb\left[ \epsilon-\Delta_\epsilon < L_{P^{\pi_e}}(\hat{C}_{\tilde{\tau}}) \leq \epsilon \right] > 1-\delta - \frac{C}{\sqrt{n}}, %<  1-\delta + \frac{C_2}{\sqrt{n}}.
	\end{equation}
	where $C=\frac{7B}{\sqrt{\gamma(\epsilon-\Delta_\epsilon)(1-\epsilon+\Delta_\epsilon)}} + \frac{(\sqrt{-2\log\delta}+1)B}{2\Delta_\epsilon\sqrt{\gamma} } + (1-\delta)(\sqrt{\lfloor m_1/\gamma \rfloor B}+\frac{B}{2}) $ and $m_1=m_0\vee \frac{\log\delta}{-2\Delta_\epsilon^2}$.
\end{thm}

Intuitively, a larger $\gamma$ and a smaller $B$ lead to higher confidence, as they correspond to a larger calibration set. Moreover, pursuing high efficiency may come at the cost of the confidence guarantee, as $C$ increases with decreasing $\Delta_\epsilon$. Nevertheless, Theorem~\ref{thm-bothsides} and \ref{lem-upperbound} show that, for arbitrarily small $\Delta_\epsilon$, $\epsilon-\Delta_\epsilon < L_{P^{\pi_e}}(\hat{C}_{\tilde{\tau}})\leq\epsilon$ 
with confidence approaching the target level $1-\delta$. This is also numerically demonstrated by a simulation study shown in Figure~\ref{fig-1}.

\begin{figure}
	\centering
	\includegraphics[width=0.40\linewidth]{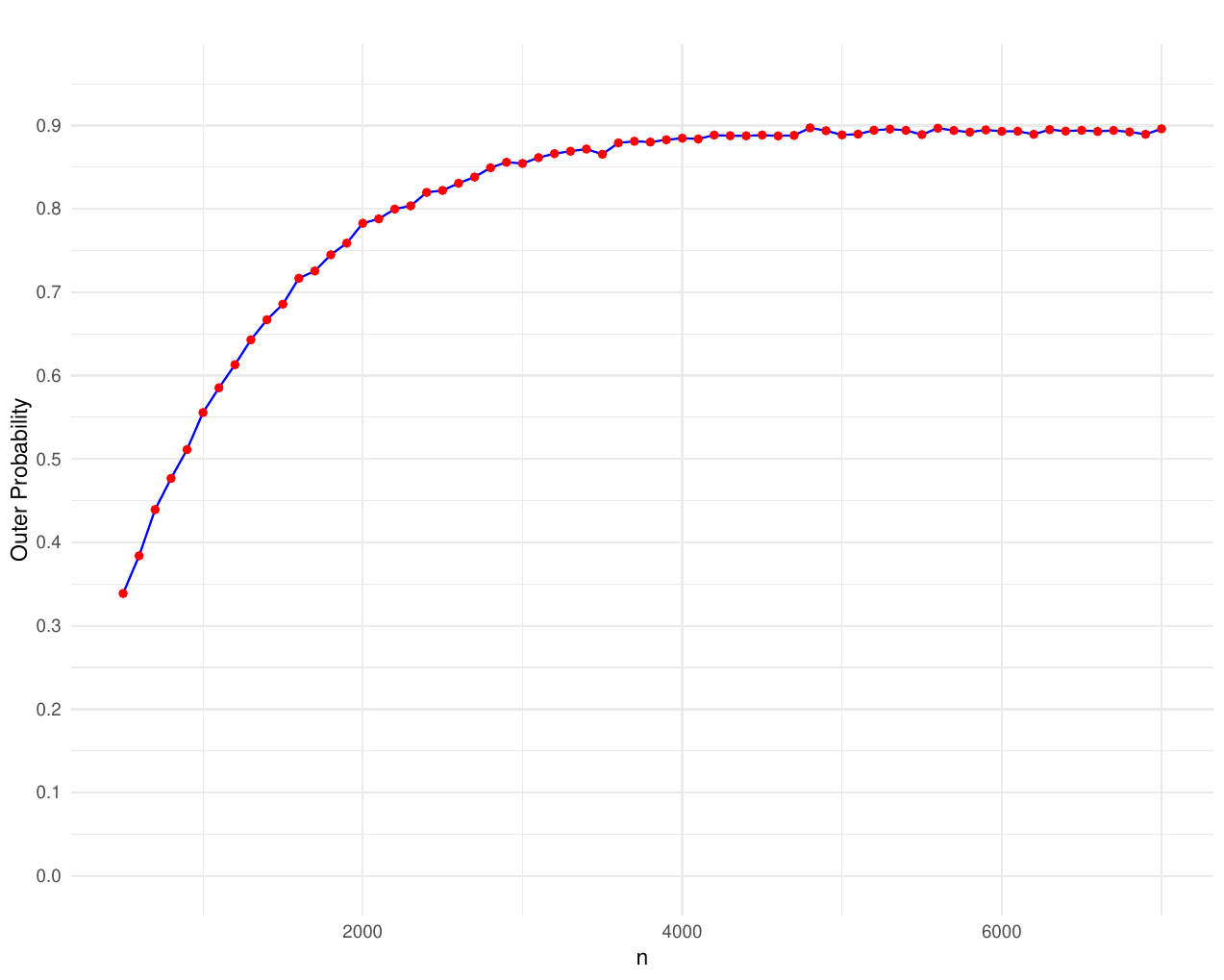}\qquad
	\includegraphics[width=0.40\linewidth]{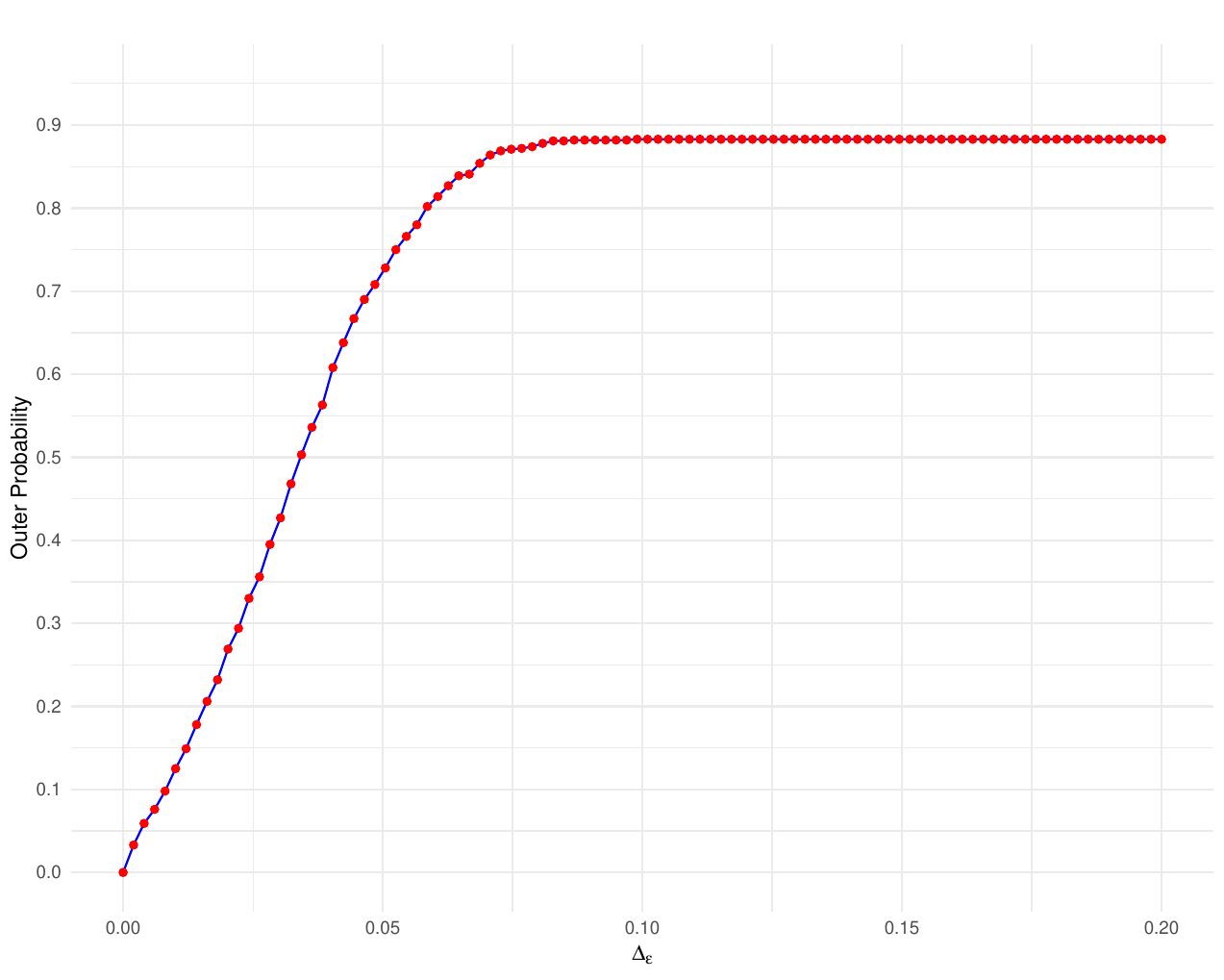}
	{\centering\caption{ Empirical $\mathbb{P}[\epsilon - \Delta_\epsilon < L_{P^{\pi_e}}(\hat{C}_{\tilde{\tau}}) \leq \epsilon]$.}\label{fig-1}}
	\caption*{\small {\bf Annotation:} The simulation was conducted under $\epsilon = 0.2$ and $\delta = 0.1$.  The left is for varying sample sizes $n$ with fixed $\Delta_\epsilon = 0.05$  and the right for varying $\Delta_\epsilon$ with fixed $n = 2000$. For each $n$, 10,000 test points were generated to evaluate $L_{P^{\pi_e}}(\hat{C}_{\tilde{\tau}})$, and 10,000 simulation runs were conducted to estimate the empirical probability.}
	
\end{figure}

In particular, by replacing $\Delta_\epsilon$ with a vanishing sequence $\Delta_n$, we can obtain the following corollary immediately.

\begin{cor}\label{cor-bothsides}
	Assume the conditions in Theorem \ref{lem-upperbound}. For a positive sequence $\Delta_n$ such that $\Delta_n=o(1)$ and $1/(\sqrt{n}\Delta_n)=o(1)$, there exists $C>0$ such that
	$$
	\Pb\left[ \epsilon-\Delta_n < L_{P^{\pi_e}}(\hat{C}_{\tilde{\tau}}) \leq \epsilon \right] > 1 - \delta - \frac{C}{\sqrt{n}\Delta_n}
	\hbox{
		and }
	\Pb\left[ \epsilon < L_{P^{\pi_e}}(\hat{C}_{\tilde{\tau}}) \leq \epsilon + \Delta_n \right] > \delta - \frac{C}{\sqrt{n}\Delta_n}.
	$$
\end{cor}

Corollary \ref{cor-bothsides} shows that $\hat{C}_{\tilde{\tau}}$ achieves exact $1-\epsilon$ coverage with probability tending to $1$. However, PIs satisfying exact $1-\epsilon$ marginal coverage are not unique, and an ideal interval should take the form
\begin{equation}\label{oracle-PI}
	C^{\mathrm{oracle}}(S_{n+1})=[q_{\epsilon_{\mathrm{lo}}}(S_{n+1}),q_{\epsilon_{\mathrm{up}}}(S_{n+1})  ],
\end{equation}
where ${q}_{\epsilon_{\mathrm{lo}}}$ and ${q}_{\epsilon_{\mathrm{up}}}$ denote the oracle conditional quantiles of $R$ given $S$ under the target policy, because it further guarantees object conditional validity with user-specified miscoverage allocation. In what follows, we establish the asymptotic equivalence between $\hat{C}_{\tilde{\tau}}$ and $C^{\mathrm{oracle}}$ under certain additional assumptions.

\begin{thm}\label{thm-asymptotic}
	Assume the consistency of quantile estimates and a regularity condition (see Appendix \ref{Proof of Theorem thm-asymptotic} for details). The output $\hat{C}_{\tilde\tau}$ of Algorithm~\ref{al-known} satisfies 
	$$
	\mathcal{L}\left( \hat{C}_{\tilde{\tau}}(S_{n+1}) \Delta C^{\mathrm{oracle}}(S_{n+1}) \right) = o_\Pb(1),	\hbox{ as }n\to \infty,
	$$
	where $\mathcal{L}(\cdot)$ denotes the Lebesgue measure and $\Delta$ is the symmetric difference operator.	
\end{thm}

\subsection{PACOPP with Unknown $\pi_b$}

The procedure above is based on a known behavior policy $\pi_b$, thereby granting access to the oracle weight function $w$. In many real-world applications, $\pi_b$ is unknown and thus must be estimated. Suppose an estimator $\hat{\pi}_b$ is learned from a subset of the observational data, say $\Ds_1$. Then the policy ratio $w$ is correspondingly estimated by $\hat{w}(s,a)=\hat{w}(s,a;\Ds_1):=\frac{\pi_e}{\hat\pi_b} (s,a)$. With the estimate $\hat{w}(s,a)$, we can follow Algorithm \ref{al-known} to generate a corresponding PI. 

This procedure is formally presented in Algorithm \ref{al-1}. The estimation of the weights inevitably leads to a degradation in coverage, as stated in Theorem~\ref{thm-est-pib}  (a proof deferred to \ref{pf-thm-est-pib}), which serves as a counterpart to Theorems \ref{thm-1}--\ref{thm-bothsides}. To present the theoretical guarantee, define 
\begin{equation}\label{def-w-error}
	\Delta_w \coloneq \Eb_{\Ds_1}  \left| \hat{w}(S,A) - w(S,A) \right|
\end{equation}
where, and thereafter, to symbolically simplify the mathematical formulae, we abuse $\Eb_{\Ds_1}[\cdot]$ to denote $\Eb[\,\cdot\mid\Ds_1]$ with $(S,A)$ independent of $\Ds_1$ and following $P_S \times \pi_b$.
\begin{algorithm}    
	\caption{PAC Off-Policy Prediction with unknown behavior policy} 
	\label{al-1} 
	\begin{algorithmic}[1] 
		\Statex \textbf{Input:} Observational data $\Ds=\{S_i, A_i, R_i\}_{i=1}^n$; test context $S_{n+1}$; PAC parameters $(\epsilon,\delta)\in(0,1)^2$; target policy $\pi_e$; policy estimation algorithm $\mathcal{A}_\mathrm{po}$; quantile prediction algorithm $\mathcal{A}_\mathrm{qu}$; quantile levels $\epsilon_{\mathrm{lo}},\epsilon_{\mathrm{up}}$ with $\epsilon_{\mathrm{up}}-\epsilon_{\mathrm{lo}}=1-\epsilon$; calibration ratio $\gamma$; function \textsc{Rs}, \textsc{Pac-Cp} in Algorithm~\ref{al-known}
		\Statex \textbf{Output:} Prediction interval $\hat{C}_{\tilde{\tau}}(S_{n+1})$
		\State \quad Split $\Ds$ into $\Ds_1\cup\Ds_2$ with $|\Ds_2| =\lceil\gamma n \rceil $ 
		\State \quad Estimate $\hat{\pi}_b$ using $\mathcal{A}_\mathrm{po}$ on $\Ds_1$
		\State \quad $\Ds^{\mathrm{rs}}_i \gets \textsc{Rs} (\Ds_i, \hat{w}(s,a) = \frac{\pi_e}{\hat{\pi}_b}( a|s))$ for $i=1,2$
		\State \quad Train $\hat{q}_{\epsilon_{\mathrm{lo}}}, \hat{q}_{\epsilon_{\mathrm{up}}}$ using $\mathcal{A}_\mathrm{qu}$ on $\Ds^{\mathrm{rs}}_1$
		\State \quad \textbf{return} $\hat{C}_{\tilde{\tau}}(S_{n+1})$ = \textsc{Pac-Cp}($\Ds^{\mathrm{rs}}_2, S_{n+1},  \epsilon, \delta, \hat{q}_{\epsilon_{\mathrm{lo}}}, \hat{q}_{\epsilon_{\mathrm{up}}}$)
	\end{algorithmic}
\end{algorithm}

\begin{thm}\label{thm-est-pib}
	If  $\sup_{(s,a)\in\Ss\times\As}\hat{w}(s,a)<\infty\;a.s.$,
	then the output $\hat{C}_{\tilde{\tau}}$ of Algorithm \ref{al-1} satisfies
	\begin{equation}\label{thm-est-pib-eq1}
		\Pb\left[ L_{P^{\pi_e}}(\hat{C}_{\tilde{\tau}})\leq\epsilon+ \Delta_w \right]\geq 1-\delta.
	\end{equation}
	In addition, if all $\tau_i$ have no ties almost surely, then for any $\Delta_\epsilon\in (0,\epsilon )$,
	\begin{equation}\label{thm-est-pib-eq2}
		\Pb\left[ \epsilon- \Delta_w-\Delta_\epsilon<L_{P^{\pi_e}}(\hat{C}_{\tilde{\tau}})\leq\epsilon + \Delta_w \right]\geq 1-\delta-\frac{C}{\sqrt{n}},
	\end{equation}
	for some positive constant $C$.
\end{thm}

Note that if we define the error as $\Delta_w = \Eb_{\Ds_1}  \left| \frac{\hat{w}(S,A)}{\Eb_{\Ds_1}[\hat{w}(S,A)]} - w(S,A) \right|$ instead of \eqref{def-w-error}, then the coverage degradation would be $\Delta_w/2$. It coincides with the marginal coverage loss observed by \cite{lei2021conformal} and \cite{taufiq2022conformal} for weighted CP method with estimated weight function. 

Generally, $\pi_b$ is estimated using certain learning algorithms and, thus, theoretical guarantees for bounding the error term (\ref{def-w-error}) remain unknown. We would like to note that, if the MLE approach is employed to estimate $\pi_b$, Theorem \ref{thm-est-pib} can be adapted to the following Theorem \ref{thm-MLE} that provides an explicit coverage loss under suitable conditions (proof in Appendix~\ref{pf-thm-MLE}) and thus shows that Algorithm \ref{al-1} achieves the required PAC validity asymptotically. 

\begin{thm}\label{thm-MLE}
	Assume access to a finite policy class $\Pi=\{\pi:\Ss\to\Delta(\As)\}$, such that $\pi_b\in\Pi$ and $\sup_{(s,a)\in\Ss\times\As}\frac{\pi_e}{\pi}(a|s)\leq B_{\Pi}<\infty$ for all $\pi\in\Pi$. If  the behavior policy is estimated by $
	\hat{\pi}_b=\underset{\pi\in\Pi}{\arg\max}\sum_{(S_i,A_i,R_i)\in\Ds_1}\log \pi(A_i,S_i),
	$
	then, for any $\delta_0\in(0,1)$, the output $\hat{C}_{\tilde{\tau}}$ of Algorithm~\ref{al-1} satisfies
	$$
	\Pb\left[ L_{P^{\pi_e}}(\hat{C}_{\tilde{\tau}})\leq\epsilon+  2B_{\Pi}\sqrt{\frac{2\log(|\Pi|/\delta_0)}{\lfloor(1-\gamma)n\rfloor}} \right]\geq (1-\delta_0)(1-\delta).
	$$
\end{thm}

%%%%%%%%%%%%%%%%%%%%%%%%%%%%%%%%%%%%%%%%%%%%%%%%%%%%%%%%%%%%%%%%%%%%%%%%%%%%%%%%%%%%%%%%
%%%%%%%%%%%%%%%%%%%%%%%%%%%%%%%%%%%%%%%%%%%%%%%%%%%%%%%%%%%%%%%%%%%%%%%%%%%%%%%%%%%%%%%%

\section{Synthetic Data Experiments}\label{sec4}
Due to the absence of established baselines for our problem, we compare Algorithm~\ref{al-1} (PACOPP) with the following two, both of which are distribution-free off-policy prediction algorithms capable of handling continuous action spaces.	
\subsection{Reference Algorithms.}	
\paragraph{(1) Conformal Off-Policy Prediction (COPP).} 
COPP was introduced in \citet{taufiq2022conformal} and has been demonstrated to achieve marginally valid coverage. Using a estimated behavior policy $\hat{\pi}_b$ and a estimated reward distribution $\hat{P}_R$, both trained on $\mathcal{D}_1$, COPP estimates the weight (\ref{COPP-weight}) via a Monte Carlo approach, that is, $\hat{w}(s,r) = \sum_{i=1}^{h} \hat{P}_R(r|s,a_i^e) /\sum_{i=1}^{h} \hat{P}_R(r|s,a_i)$, where $a_i\sim\hat{\pi}_b(\cdot|s), a_i^e\sim\pi_e(\cdot|s)$ and $h$ is the number of Monte Carlo samples. COPP constructs PIs based on the weighted CP framework. That is, it outputs $(s,r)$ pairs whose non-conformity scores lie below the $1-\epsilon$ quantile of the weighted empirical distribution: $\hat{F}_{n_2}^{s,r}\coloneq\sum_{\Ds_2}p_i^{\hat{w}}(s,r)\delta_{\tau_i}+p_{n_2+1}^{\hat{w}}(s,r)\delta_{\infty}$, where $p_i^{\hat{w}}(s,r)\coloneq\frac{\hat{w}(S_i,R_i)}{\sum_{\Ds_2}\hat{w}(S_i,R_i)+\hat{w}(s,r)}$, and $p_{n_2+1}^{\hat{w}}(s,r)\coloneq\frac{\hat{w}(s,r)}{\sum_{\Ds_2}\hat{w}(S_i,R_i)+\hat{w}(s,r)}$.		
\paragraph{(2) Conformal Off-Policy Prediction with Rejection Sampling (COPP-RS).}
To better address the question of "why use PAC PIs instead of marginally valid PIs," we designed an algorithm COPP-RS, which, like PACOPP, employs rejection sampling to overcome distribution shift. The only difference is that COPP-RS directly uses the $1-\epsilon$ empirical quantile of the non-conformity scores as the threshold. Specifically, denote by $\tau^{(1-\epsilon)}$ the $1-\epsilon$ quantile of the distribution $\sum_{(S_i,R_i)\in\Ds^{\mathrm{rs}}_2} \frac{1}{M+1}\delta_{\tau_i} + \frac{1}{M+1}\delta_{\infty}$. COPP-RS then outputs the PI $\hat{C}_{\tau^{(1-\epsilon)}} (S_{n+1}) = [\hat{q}_{\epsilon_{\mathrm{lo}}}(S_{n+1}) - \tau^{(1-\epsilon)}, \hat{q}_{\epsilon_{\mathrm{up}}}(S_{n+1}) + \tau^{(1-\epsilon)}] $. It is straightforward to verify that $\hat{C}_{\tau^{(1-\epsilon)}}$ also achieves marginal $1-\epsilon$ coverage under the conditions in Theorem \ref{thm-1}.

\subsection{Computational Details.}
We use an experimental setup with continuous action space similar to the one described in \citet{taufiq2022conformal}, depicted as follows. The code for the experiment was released at \url{https://anonymous.4open.science/r/PACOPP-CB2D}.

(1) The contexts and actions were sampled according to $
S \sim \Normal(0,4)$ and $ A \mid s \sim \Normal(s/4,4)
$ and,
given the context-action pair $(s,a)$, the reward $R$ was generated from a two-component Gaussian mixture distribution
$
P_R(\,\cdot\mid s,a) =			\Normal(s+a, 1)*0.2+\Normal(s+a, 16)* 0.8.	$

(2) A total of 2000 samples were generated and randomly split into two equal halves: one for training and the other for calibration. Neural networks (NNs) were trained on the training set to estimate the behavior policy  $\hat{\pi}_b$. Separately, the quantile functions $\hat{q}_{\epsilon/2}$ and $\hat{q}_{1-\epsilon/2}$ were learned by fitting NNs using the pinball loss. For the COPP algorithm, we further approximated the reward distribution $P_R$ using a misspecified conditional Gaussian model $\hat{P}_R(\, \cdot \mid s,a) = \Normal(\mu(s,a),\sigma(s,a))$, where both $\mu(s,a)$ and $\sigma(s,a)$ are NNs. Finally, the target policy was defined as
$
\pi_e(\, \cdot \mid s) = \Normal(s/4,1).
$

(3)	In each simulation run, 10,000 test data points are generated from the target distribution to evaluate the coverage rate. The nominal miscoverage level is set to $\epsilon = 0.2$, and a total of 1000 independent simulations are conducted. 

\subsection{Results.}
Figure~\ref{fig-2} reports the coverage and interval lengths of different methods. For the PACOPP algorithm, we set $\delta = 0.5$, $0.25$, $0.1$ and $0.01$, indicated by PAC-0.5, PAC-0.25, PAC-0.1 and PAC-0.01, respectively. The main findings were summarized as follows.

First, COPP generally failed to achieve the nominal coverage due to the misspecification in the estimated conditional density model. In contrast, the other two methods based on rejection sampling (COPP-RS and PACOPP) attained the desired coverage in most cases. Second, compared to COPP-RS, PACOPP offered additional control over the probability of attaining nominal coverage. Notably, as the confidence level increased, the interval length grew only moderately. Even under the stringent requirement of $\delta = 0.01$, the PIs produced by PACOPP were not overly conservative.

\begin{figure}
	\centering
	\includegraphics[width=0.40\linewidth]{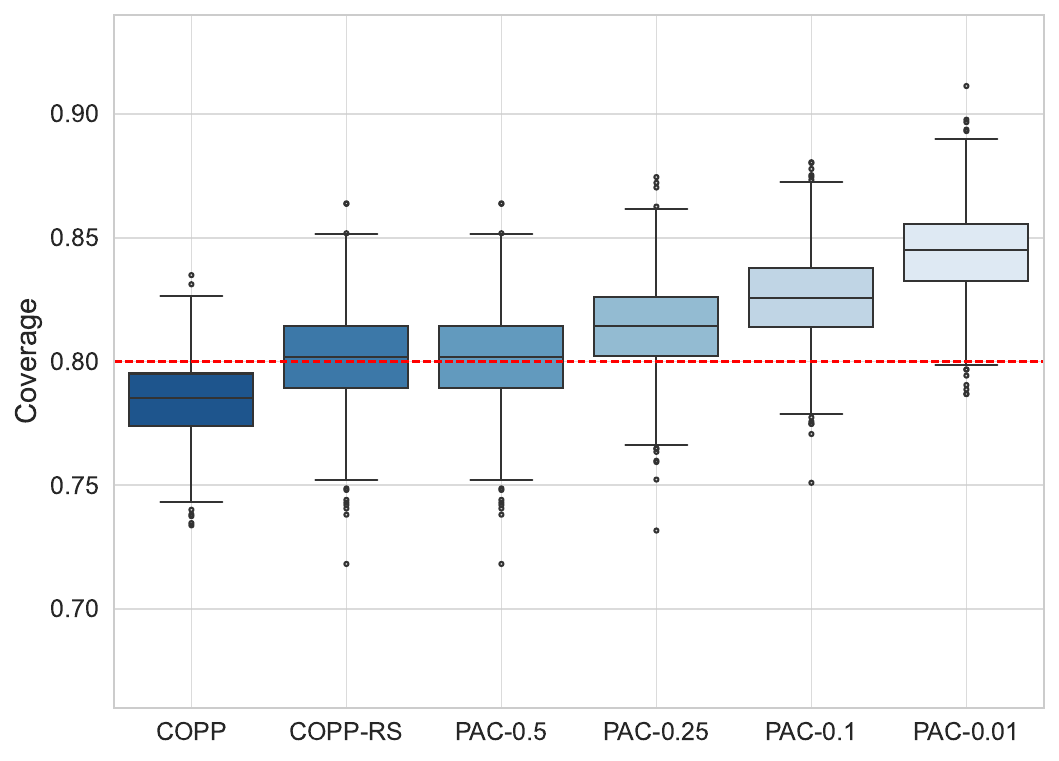}\quad
	\includegraphics[width=0.40\linewidth]{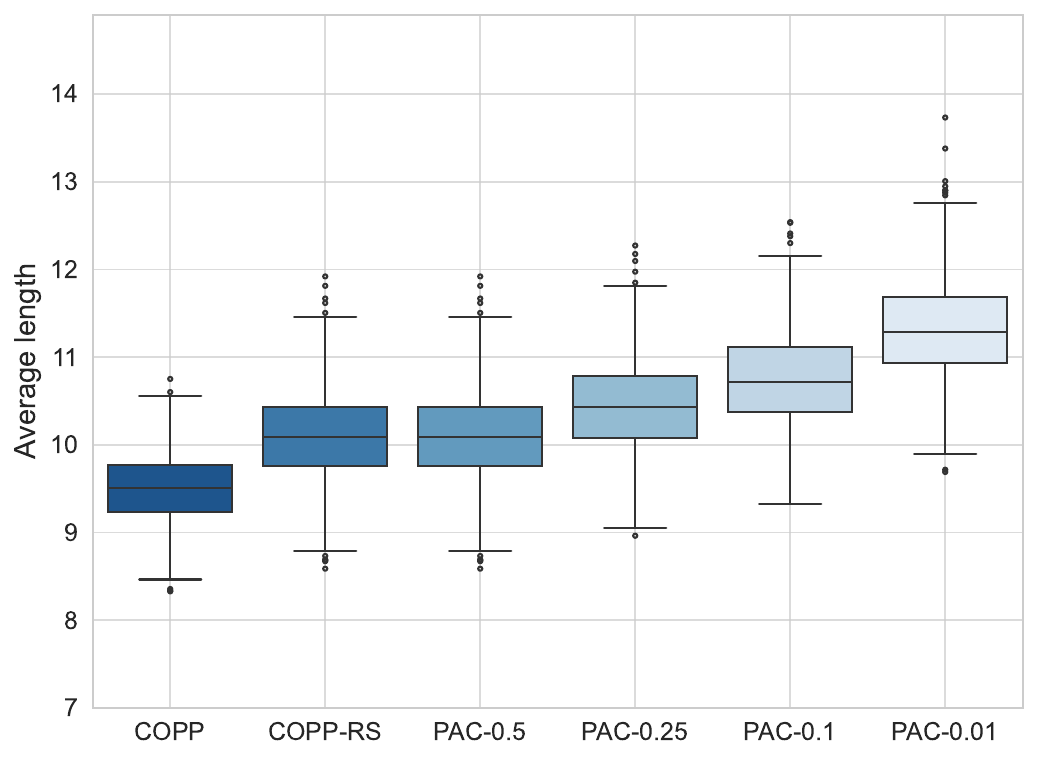}
	\caption{Empirical coverages and average lengths of prediction intervals based on COPP, COPP-RS and PACOPP with $\delta = 0.5, 0.15 , 0.1$ and $0.01$.  The nominal level is $80\%$.}
	\label{fig-2}
\end{figure}		

Compared to marginally valid methods, PACOPP provided a principled mechanism for navigating the trade-off between confidence and interval efficiency. Therefore, in safety-critical scenarios where a high confidence guarantee is prioritized even at a price of slight conservativeness, PACOPP may be preferable to marginally valid alternatives.

%%%%%%%%%%%%%%%%%%%%%%%%%%%%%%%%%%%%%%%%%%%%%%%%%%%%%%%%%%%%%%%%%%%%%%%%%%%%%%%%%%%%%%%%
%%%%%%%%%%%%%%%%%%%%%%%%%%%%%%%%%%%%%%%%%%%%%%%%%%%%%%%%%%%%%%%%%%%%%%%%%%%%%%%%%%%%%%%%

\section{Related Work}\label{sec5}

{\bf Off-policy evaluation.}
OPE is one of the most fundamental topics in Reinforcement learning \citep{sutton2018reinforcement} and has been extensively studied, resulting in a vast body of literature. The primary challenge in OPE arises from the distribution shift in rewards, induced by the discrepancy between the behavior and target policies. Current methods, which typically focus on estimating the expected reward (policy value), are broadly categorized into three main approaches: (i) importance sampling \citep{precup2000eligibility,liu2018breaking, schlegel2019importance}, known for its unbiased nature but susceptible to high variance; (ii) direct methods \citep{thomas2016data, le2019batch,shi2022statistical}, which directly learn the model before policy evaluation, potentially introducing bias but offering lower variance; and (iii) doubly robust methods \citep{dudik2011doubly, jiang2016doubly,kallus2020double}, which combine the first two approaches to achieve more robust estimators. For an extensive review, we refer readers to \cite{uehara2022review}.   %MIS marginalized importance weights

In addition to point estimates for the value of the target policy, less attention has been paid to interval estimates of the policy value for uncertainty quantification. To provide confidence regarding the accuracy of these estimates, \cite{thomas2015high} proposed high confidence off-policy evaluation, which derives a lower confidence bound on the target policy's value by applying concentration inequalities to importance sampling estimates. Other approaches, such as bootstrap \citep{hanna2017bootstrapping}, kernel Bellman loss \citep{feng2020accountable} and empirical likelihood methods \citep{dai2020coindice}, have also been employed to derive confidence intervals. Nevertheless, all of these methods focus on the average effect of the target policy and do not account for the variability in the reward itself.

{\bf Conformal prediction.}
CP has gained popularity in OPE due to its ability to construct distribution-free PIs with finite-sample guarantees and account for individual effects. The application of CP to OPE originated from \cite{tibshirani2019conformal}, who developed the ``weighted conformal prediction'' method that extends standard CP to covariate shift settings, in which the covariate distributions of test and training data differ, while the conditional distributions remain the same. This method offers an valuable approach to addressing distribution shift and was subsequently applied to OPE in contextual bandits \citep{taufiq2022conformal} (COPP) and Markov decision processes \citep{foffano2023conformal}.

As discussed in Section \ref{sec4}, these direct application of weighted CP in OPE requires estimating the conditional probability densities of rewards, and may underperform if the model is misspecified. Additionally, COPP directly estimates the conditional quantiles from observational data. As a result, the algorithm essentially calibrates intervals constructed from the estimated conditional quantiles under the behavior policy. Even if the quantile estimation algorithm is consistent, the resulting PIs will not converge to the oracle intervals defined in (\ref{oracle-PI}). In contrast, PACOPP ensures this convergence property (Theorem \ref{thm-asymptotic}). Furthermore, since the weight (\ref{COPP-weight}) depends on both $s$ and $r$, COPP must use a grid of potential values for $R_{n+1}$ corresponding to each $S_{n+1}$ when generating the final PI, introducing additional computational overhead. By contrast, PACOPP explicitly outputs PIs without this extra burden.

In \cite{zhang2023conformal}, the authors avoid model estimation by selecting subsamples from observational data where the action matches the pseudo action generated by a designed auxiliary policy. These subsamples preserve the same conditional distribution as the target population, enabling the use of weighted CP for PI construction. However, their method is inherently limited, as it is only applicable when the action space is discrete. In contrast, our approach, PACOPP, is applicable to continuous action spaces, such as the doses of medication administered in precision medicine.

As discussed in Section \ref{sec3}, although split (inductive) CP has been shown in \cite{vovk2013conditional} to automatically achieve training conditional validity in a PAC type, it is not applicable in our context due to its sample size requirement. In \cite{park2020pac}, the authors proposed an modified version of split CP that constructs confidence sets for deep neural networks with finite-sample PAC validity. This method was further extended to handle covariate shift settings in \cite{park2022pac}, where it was adjusted using Clopper-Pearson upper bounds to handle cases when the importance weights (the likelihood ratio of covariate distributions) are unknown, but confidence intervals for these weights are available. We note that their candidate intervals are constructed in a different form, and the associated threshold is determined by solving an optimization problem, as mentioned in Section~\ref{sec3}, In contrast, the construction in PACOPP admits an explicit threshold of the form (\ref{def-tilde_tau}). Most importantly, we theoretically justified the efficiency of our method, both in finite samples and asymptotically, whereas their approach lacks such theoretical guarantees.

%%%%%%%%%%%%%%%%%%%%%%%%%%%%%%%%%%%%%%%%%%%%%%%%%%%%%%%%%%%%%%%%%%%%%%%%%%%%%%%%%%%%%%%%
%%%%%%%%%%%%%%%%%%%%%%%%%%%%%%%%%%%%%%%%%%%%%%%%%%%%%%%%%%%%%%%%%%%%%%%%%%%%%%%%%%%%%%%%

\section{Conclusion}\label{sec6}

In this paper, we proposed a novel algorithm, PACOPP, for constructing reliable prediction intervals in the OPE task of contextual bandits. PACOPP provides distribution-free finite-sample validity and is proven to be asymptotically efficient under mild assumptions. Distinct from existing methods, PACOPP is the first to achieve PAC validity, enabling explicit control over the confidence level of achieving the target coverage.

We addressed distribution shift through rejection sampling with the cost of reduction of available sample size, resulting in lower data utilization and ultimately leading to more conservative PIs. Enhancing the data utilization in this setting remains a challenge that requires further research efforts. Additionally, extending PACOPP to general sequential decision-making scenarios beyond contextual bandits would be an promising direction for future research.

%%%%%%%%%%%%%%%%%%%%%%%%%%%%%%%%%%%%%%%%%%%%%%%%%%%%%%%%%%%%%%%%%%%%%%%%%%%%%%%%%%%%%%%%
%%%%%%%%%%%%%%%%%%%%%%%%%%%%%%%%%%%%%%%%%%%%%%%%%%%%%%%%%%%%%%%%%%%%%%%%%%%%%%%%%%%%%%%%

%{The authors declare that there is no confict of interest. This work was supported by National Key R\&D Program of China (Nos. 2021YFA1000100 and 2021YFA1000101) and Natural Science Foundation of China (No. 72371103).}

%%%%%%%%%%%%%%%%%%%%%%%%%%%%%%%%%%%%%%%%%%%%%%%%
%\nocite{park2021pac}
%\nocite{}                      
%\nocite{*}                    
%\bibliographystyle{plain}      
\bibliographystyle{plainnat}    
\bibliography{bibliography}
%%%%%%%%%%%%%%%%%%%%%%%%%%%%%%%%%%%%%%%%%%%%%%%%		

\newpage

\appendix
\section{Proofs of the Main Results}

\subsection{Proof of Proposition \ref{pro-1}}\label{pf-pro-1}
\begin{proof}
	Thanks to the independence between $\cal D$ and $\{V_i,i=1,2\dots,n\}$,  $$\Pb\left(V_i\leq \frac{1}{B}w(S_i,A_i)\right)=\int_\Ss\int_\As\frac1B\frac{\pi_e}{\pi_b}( a|s)\pi_b( a|s)P_S( s ) \drm a \drm s =\frac{1}{B},i\in[n].$$ Thus, the size $N_{\mathrm{rs}}\sim \Bin(n,\frac{1}{B})$.
	
	For any $\{(s_j,r_j)\}_{j=1}^{N_{\mathrm{rs}}}\in(\Ss\times\Rb)^{N_{\mathrm{rs}}}$, conditional on the event $\{N_{\mathrm{rs}}=n_{\mathrm{rs}}\}$,
	\begin{align*}
		&\Pb(Z_j=(s_j, r_j),\forall j\in[N_{\mathrm{rs}}]\mid N_{\mathrm{rs}}=n_{\mathrm{rs}})\\
		=&\frac{1}{\Pb(N_{\mathrm{rs}}=n_{\mathrm{rs}})}\Pb\left(\exists\;\sigma_1<\cdots<\sigma_{n_{\mathrm{rs}}}\in[n]\; s.t.\;(S_{\sigma_j},R_{\sigma_j})=(s_j, r_j), V_{\sigma_j}\leq \frac{1}{B}w(S_{\sigma_j},R_{\sigma_j}),\right.\\
		&\left.\qquad\qquad\qquad\qquad\qquad\qquad \forall j\in[n_{\mathrm{rs}}], \text{and }V_i> \frac{1}{B}w(S_i,A_i),\forall i\in[n]\setminus\{\sigma_1,\ldots,\sigma_{n_{\mathrm{rs}}}\}\right)\\
		=&\frac{1}{\binom{n}{n_{\mathrm{rs}}}B^{-n_{\mathrm{rs}}}(1-\frac{1}{B})^{n-n_{\mathrm{rs}}}}\sum_{\sigma_1<\cdots<\sigma_{n_{\mathrm{rs}}}}
		\Big[B^{-n_{\mathrm{rs}}}(1-\frac{1}{B})^{n-n_{\mathrm{rs}}}
		\\ &\qquad\qquad\qquad\qquad\qquad\qquad\times\prod_{j\in[n_{\mathrm{rs}}]}\Pb\Big((S_{\sigma_j},R_{\sigma_j})=(s_j, r_j)\;|\;V_{\sigma_j}\leq \frac{1}{B}w(S_{\sigma_j},R_{\sigma_j})\Big)\Big],
	\end{align*}
	where the summation is taken over all possible choices of $\sigma_1,\ldots,\sigma_{n_{\mathrm{rs}}}$. By the fact that for any $i$,
	\begin{align*}
		&\Pb\left((S_i,R_i)=(s, r)\mid V_i\leq \frac{1}{B}w(S_i,R_i)\right)
		\\&\qquad=B\int_\As\frac{1}{B}\frac{\pi_e}{\pi_b}( a|s) P_S( s)\pi_b( a|s)P_{R}( r|s,a)\drm a  = P^{\pi_e}( s,r),
	\end{align*}
	we then have
	$$
	\Pb(Z_j=(s_j, r_j),\forall j\in[N_{\mathrm{rs}}]\mid N_{\mathrm{rs}}=n_{\mathrm{rs}})=\prod_{j\in[n_{\mathrm{rs}}]}P^{\pi_e}( s_j, r_j).
	$$
	The proof is complete.
\end{proof}

\subsection{Proof of Theorem \ref{thm-1}}\label{pf-thm-1}
%Here, we provide a novel and concise proof for Theorem \ref{thm-1} without imposing any distributional assumptions, which differs from the approach in \citep{park2020pac}.
\begin{proof}
	We first consider the probability of $\{L_{P^{\pi_e}}(\hat{C}_{\tau})\leq\epsilon\}$ conditional on the event that the RS procedure generates $N_{\mathrm{rs}}=n_{\mathrm{rs}}$ samples. We omit the trivial cases of $n_{\mathrm{rs}}=0$, when $\hat{C}_\tau(s)$ can be set to $\Rb$. For nontrivial cases, Proposition \ref{pro-1} enables us to interpret the conditional probability as
	$$
	\Pb[L_{P^{\pi_e}}(\hat{C}_\tau)\leq\epsilon\mid N_{\mathrm{rs}}=n_{\mathrm{rs}}]=\Pb_{\Ds_{\mathrm{rs}}\sim(P^{\pi_e})^{n_{\mathrm{rs}}}}[L_{P^{\pi_e}}(\hat{C}_\tau)\leq\epsilon].
	$$
	Now, for any partition of $\Ds_{\mathrm{rs}}$ and any realization of the set $\Ds^{\mathrm{rs}}_1$, the parameterized interval $\hat{C}_\tau(s)$ for a fixed $s$ is nonrandom after training $\hat{q}_{\epsilon_{\mathrm{lo}}}$ and $\hat{q}_{\epsilon_{\mathrm{up}}}$, and depends only on $\tau$. It can be easily verified from the definition (\ref{def-Ctau}) that the miscoverage $L_{P^{\pi_e}}(\hat{C}_\tau)=\Pb_{(S_{n+1},R_{n+1})\sim P^{\pi_e}}(R_{n+1}\notin \hat{C}_\tau(S_{n+1}))$, as a function of $\tau$, is monotonically decreasing and right-continuous. Next, we define
	\begin{equation}\label{def-pf-thm1-tau*}
		\tau^*\coloneq\inf\{\tau\in\Rb:L_{P^{\pi_e}}(\hat{C}_\tau)\leq\epsilon\},
	\end{equation}
	and let $\{\alpha_j\}_{i=1}^\infty$ be a positive sequence such that $\alpha_j\downarrow0$. Denote by $m=\lceil\gamma n_{\mathrm{rs}}\rceil$ the size of $\Ds^{\mathrm{rs}}_2$. For $\tilde{\tau}=\tau_{(m-k(m,\epsilon,\delta))}$ with $k(m,\epsilon,\delta)\neq -1$, the right-continuity implies that the event
	$$\{L_{P^{\pi_e}}(\hat{C}_{\tilde{\tau}})>\epsilon\}=\{\tilde{\tau}<\tau^*\}=\bigcup_{j=1}^\infty\{\tilde{\tau}\leq\tau^*-\alpha_j\}.$$
	Since $\tau_i$ represents the minimal $\tau$ such that $R_i\in\hat{C}_\tau(S_i)$ in $\Ds^{\mathrm{rs}}_2$,
	\begin{align*}
		\{\tilde{\tau}\leq \tau^*-\alpha_j\}&=\{ \exists \text{ at most } k(m,\epsilon,\delta) \text{ indices }i\;s.t.\;\tau_i > \tau^*-\alpha_j\}\\%&\Longrightarrow\{ \exists \text{ at most } k(m,\epsilon,\delta) \text{ indices }i\;s.t.\;\tau_i > \tau^*-\alpha_j\}\\
		&=\{ \exists \text{ at most } k(m,\epsilon,\delta) \text{ indices }i\;s.t.\;R_i \notin \hat{C}_{\tau^*-\alpha_j}(S_i)\}.
	\end{align*}
	As $\Ds^{\mathrm{rs}}_2\sim(P^{\pi_e})^m$, each sample in $\Ds^{\mathrm{rs}}_2$ independently satisfies $R_i \notin \hat{C}_{\tau^*-\alpha_j}(S_i)$ with probability $L_{P^{\pi_e}}(\hat{C}_{\tau^*-\alpha_j})$, and we have that $L_{P^{\pi_e}}(\hat{C}_{\tau^*-\alpha_j})>\epsilon$ by the definition of $\tau^*$. Then, it holds
	$$  \Pb_{\Ds^{\mathrm{rs}}_2\sim(P^{\pi_e})^m}(\tilde{\tau}\leq\tau^*-\alpha_j)
	\leq  F_{\text{Bin}(m,L_{P^{\pi_e}}(\hat{C}_{\tau^*-\alpha_j}))}(k(m,\epsilon,\delta))
	\leq  F_{\text{Bin}(m,\epsilon)}(k(m,\epsilon,\delta))
	\leq \delta,$$
	where the second inequality follows from the fact that, for a fixed $k$, the c.d.f. $F_{\text{Bin}(m,\varepsilon)}(k)\coloneq\sum_{i=0}^k\binom{m}{i}\varepsilon^i(1-\varepsilon)^{m-i}$ is decreasing w.r.t. $\varepsilon\in[0,1]$, and the last inequality follows from the definition of $k(m,\epsilon,\delta)$. Together with the continuity of measures, we have
	\begin{equation}\label{eq-pf-thm1}
		\Pb_{\Ds^{\mathrm{rs}}_2\sim(P^{\pi_e})^m}[L_{P^{\pi_e}}(\hat{C}_{\tilde{\tau}})>\epsilon]=\lim_{j\to\infty}\Pb_{\Ds^{\mathrm{rs}}_2\sim(P^{\pi_e})^m}(\tilde{\tau}\leq \tau^*-\alpha_j)\leq\delta,
	\end{equation}
	which also holds if $k(m,\epsilon,\delta)=-1$, in which case $L_{P^{\pi_e}}(\hat{C}_{\tilde{\tau}})=0$.
	
	Finally, since (\ref{eq-pf-thm1}) is true for any partition and realization of $\Ds^{\mathrm{rs}}_1$, we can marginalize to obtain
	$$
	\Pb_{\Ds_{\mathrm{rs}}\sim(P^{\pi_e})^{n_{\mathrm{rs}}}}[L_{P^{\pi_e}}(\hat{C}_{\tilde{\tau}})\leq\epsilon]\geq1-\delta.
	$$
	By multiplying both sides by $\Pb(N_{\mathrm{rs}}=n_{\mathrm{rs}})$ and summing over all possible $n_{\mathrm{rs}}$, the proof is complete.
\end{proof}

\subsection{Proof of Theorem \ref{lem-upperbound}}\label{pf-thm-bothsides}
The proof of Theorem \ref{lem-upperbound} relies on the well-known Berry-Esseen inequality, presented in Lemma \ref{lem-BerryEsseen} (Theorem 3.4.17 in \cite{durrett2019probability}). For convenience of reference, we list it as follows. 
\begin{lem}[Berry-Esseen inequality]\label{lem-BerryEsseen}
	Let $X_1,X_2,\ldots$ be i.i.d. with $\Eb X_i=0$, $\Eb X_i^2=\sigma^2$ and $\Eb|X_i|^3=\rho<\infty$. If $F_n(x)$ is the distribution function of $\sum_{i=1}^nX_i/(\sigma\sqrt{n})$ and $\Phi(x)$ is the standard normal
	distribution function, then it holds for all $n$ that
	$$
	\sup_{x\in\Rb}\left|F_n(x) -\Phi(x) \right|\leq\frac{3\rho}{\sigma^3\sqrt{n}}.
	$$
\end{lem}
\begin{proof}
	Similar to the proof of Theorem \ref{thm-1}, we first show that there exists some $C_1>0$, s.t.
	\begin{equation}\label{eq-pf-lem-1}
		\Pb\left[ L_{P^{\pi_e}}(\hat{C}_{\tilde{\tau}})\leq\epsilon \mid N_{\mathrm{rs}}=n_{\mathrm{rs}}, \Ds^{\mathrm{rs}}_1\right]=\Pb_{\Ds^{\mathrm{rs}}_2\sim(P^{\pi_e})^m}\left[ L_{P^{\pi_e}}(\hat{C}_{\tilde{\tau}})\leq\epsilon\right] < 1-\delta+\frac{C_1}{\sqrt{m}},
	\end{equation}
	provided $k(m,\epsilon,\delta)\neq -1$. Recall $\tilde{\tau}=\tau_{(m-k(m,\epsilon,\delta))}$. For $\tau^*$ defined in (\ref{def-pf-thm1-tau*}),
	\begin{align*}
		\{L_{P^{\pi_e}}(\hat{C}_{\tilde{\tau}})\leq\epsilon\}=\{\tilde{\tau}\geq\tau^*\}&=\{ \exists \text{ at least } k(m,\epsilon,\delta)+1 \text{ indices }i\;s.t.\;\tau_i \geq \tau^*\}\\
		&\subset\{ \exists \text{ at least } k(m,\epsilon,\delta) \text{ indices }i\;s.t.\;\tau_i > \tau^*\}\\
		&=\{ \exists \text{ at least } k(m,\epsilon,\delta) \text{ indices }i\;s.t.\;R_i \notin \hat{C}_{\tau^*}(S_i)\},
	\end{align*}
	since there are almost surely no ties. Together with the fact $L_{P^{\pi_e}}(\hat{C}_{\tau^*})\leq \epsilon$, we have
	\begin{align*}
		&\Pb_{\Ds^{\mathrm{rs}}_2\sim(P^{\pi_e})^m}[L_{P^{\pi_e}}(\hat{C}_{\tilde{\tau}})\leq\epsilon]\leq 1-F_{\text{Bin}(m,L_{P^{\pi_e}}(\hat{C}_{\tau^*}))}(k(m,\epsilon,\delta)-1)\\
		&\quad\leq 1-F_{\text{Bin}(m,\epsilon)}(k(m,\epsilon,\delta)-1)\\
		&\quad=1-F_{\text{Bin}(m,\epsilon)}(k(m,\epsilon,\delta)+1)+F_{\text{Bin}(m,\epsilon)}(k(m,\epsilon,\delta)+1)
		-F_{\text{Bin}(m,\epsilon)}(k(m,\epsilon,\delta)-1)\\
		&\quad<1-\delta+F_{\text{Bin}(m,\epsilon)}(k(m,\epsilon,\delta)+1)-F_{\text{Bin}(m,\epsilon)}(k(m,\epsilon,\delta)-1),
	\end{align*}
	where the last inequality follows from the definition of $k(m,\epsilon,\delta)$. To bound the last two terms, we note that %(for clarity of notation, denote $k(m,\epsilon,\delta)$ by $k(m,\epsilon,\delta)$)
	\begin{align}
		&F_{\text{Bin}(m,\epsilon)}(k(m,\epsilon,\delta)+1)-F_{\text{Bin}(m,\epsilon)}(k(m,\epsilon,\delta)-1)\nonumber\\
		&=F_{m,\epsilon}(x_m)-\Phi(x_m)-F_{m,\epsilon}(\sqrt{m} y_m)+\Phi(y_m)+\Phi(x_m)-\Phi(y_m)\label{est-1},
	\end{align}
	where $\sigma_\epsilon:=\sqrt{\epsilon(1-\epsilon)}$, $x_m=\frac{k(m,\epsilon,\delta)+1-m\epsilon}{\sigma_\epsilon \sqrt{m}}$, $y_m=\frac{k(m,\epsilon,\delta)-1-m\epsilon}{\sigma_\epsilon \sqrt{m}}$ and $F_{m,\epsilon}$ is the distribution function of $\sum_{i=1}^mX_i/(\sigma_\epsilon\sqrt{m})$. Here $X_i$ are i.i.d and $\Pb(X_i=-\epsilon)=1 - \epsilon=1-\Pb(X_i=1-\epsilon)$. It then follows from Lemma \ref{lem-BerryEsseen} that
	\begin{align*}
		&F_{\text{Bin}(m,\epsilon)}(k(m,\epsilon,\delta)+1)-F_{\text{Bin}(m,\epsilon)}(k(m,\epsilon,\delta)-1)\\
		&\quad\leq 2\sup_{x\in\Rb}\left|F_{m,\epsilon}(x)- \Phi(x)\right|%+m\epsilon
		+\Phi(x_m)-\Phi(y_m)\\
		&\quad\leq \frac{6(\epsilon(1-\epsilon)^3+(1-\epsilon)\epsilon^3)}{\sigma_\epsilon^3}\frac{1}{\sqrt{m}}
		+\sqrt{\frac{2}{\pi}}\frac{1}{\sigma_\epsilon\sqrt{m}}\leq 7\frac{1}{\sigma_\epsilon\sqrt{m}}.
	\end{align*}
	Therefore, (\ref{eq-pf-lem-1}) holds with $C_1=7/\sigma_\epsilon$ provided $k(m,\epsilon,\delta)\neq-1$.

	Secondly, as $m= \lceil\gamma n_{\mathrm{rs}}\rceil $, we marginalize (\ref{eq-pf-lem-1}) and get that for $n_{\mathrm{rs}}\geq \lfloor m_0/\gamma\rfloor+1$, where $m_0\coloneq \log_{1-\epsilon}{\delta}$, 
	\begin{align*}%\label{eq-pf-lem-2}
		\Pb\left[ L_{P^{\pi_e}}(\hat{C}_{\tilde{\tau}})\leq\epsilon|N_{\mathrm{rs}}=n_{\mathrm{rs}} \right]&=\Eb\left[\Pb\left( L_{P^{\pi_e}}(\hat{C}_{\tilde{\tau}})\leq\epsilon \big|N_{\mathrm{rs}}=n_{\mathrm{rs}},\Ds^{\mathrm{rs}}_1 \right)\Big|N_{\mathrm{rs}}=n_{\mathrm{rs}}\right]
		\\&< 1-\delta+\frac{C_1}{\sqrt{\lceil\gamma n_{\mathrm{rs}}\rceil}}
		\leq 1-\delta+\frac{C_1}{\sqrt{\gamma n_{\mathrm{rs}}}},
	\end{align*}
	because $k(m,\epsilon,\delta)\neq -1$ for all  $m \geq m_0$.

	Finally, we prove the desired conclusion of Theorem \ref{lem-upperbound}. Note that for $n\geq \lfloor m_0/\gamma\rfloor+1$,
	\begin{align*}
		&\Pb\left[ L_{P^{\pi_e}}(\hat{C}_{\tilde{\tau}})\leq\epsilon\right]=\sum_{n_{\mathrm{rs}}=0}^{n}\Pb(N_{\mathrm{rs}}=n_{\mathrm{rs}})\Pb\left[ L_{P^{\pi_e}}(\hat{C}_{\tilde{\tau}})\leq\epsilon|N_{\mathrm{rs}}=n_{\mathrm{rs}}\right]\nonumber\\
		&<\sum_{n_{\mathrm{rs}}=0}^{\lfloor m_0/\gamma\rfloor} \Pb(N_{\mathrm{rs}}=n_{\mathrm{rs}})+
		\sum_{n_{\mathrm{rs}}=\lfloor m_0/\gamma\rfloor +1}^{n}\Pb(N_{\mathrm{rs}}=n_{\mathrm{rs}})(1-\delta+\frac{C_1}{\sqrt{\gamma n_{\mathrm{rs}}}})\nonumber\\
		&\leq 1-\delta +F_{\Bin(n,1/B)}(\lfloor m_0/\gamma\rfloor)+   \frac{C_1}{\sqrt{\gamma}} \sum_{n_{\mathrm{rs}}=\lfloor m_0/\gamma\rfloor +1}^{n} \frac{1}{\sqrt{n_{\mathrm{rs}}}}\binom{n}{n_{\mathrm{rs}}}\frac{1}{B^{n_{\mathrm{rs}}}}(1-\frac{1}{B})^{n-n_{\mathrm{rs}}}\nonumber\\
		&\leq 1-\delta +F_{\Bin(n,1/B)}(\lfloor m_0/\gamma\rfloor)+ \frac{C_1B}{\sqrt{\gamma n}} \sum_{n_{\mathrm{rs}}=\lfloor m_0/\gamma\rfloor +1}^{n}\binom{n+1}{n_{\mathrm{rs}}+1}\frac{1}{B^{n_{\mathrm{rs}}+1}}(1-\frac{1}{B})^{n-n_{\mathrm{rs}}}\nonumber%\frac{1}{\sqrt{n_{\mathrm{rs}}}} \frac{n_{\mathrm{rs}}+1}{n+1}
		\\	&\leq 1-\delta + \frac{C_1B}{\sqrt{\gamma n}}+F_{\Bin(n,1/B)}(\lfloor m_0/\gamma\rfloor)%\label{lem-upperbound-1}.
	\end{align*}
	Furthermore, by using the Hoeffding's inequality, we have that for all $n > B \lfloor m_0/\gamma \rfloor$,
	$$
	F_{\Bin(n,1/B)}(\lfloor m_0/\gamma\rfloor)\leq \exp\left\{-2n\left(\frac{1}{B}-\frac{\lfloor m_0/\gamma\rfloor}{n}\right)^2\right\}\leq \frac{C_2}{\sqrt{n}}, 
	% e^{\frac{4m_0}{B\gamma}}e^{-\frac{2n}{B^2}}.%\label{lem-upperbound-2}
	$$
	where $C_2 = \sqrt{B \lfloor m_0/\gamma \rfloor} + B/2$ follows from straightforward calculations. %yet tedious
	%Obviously, when $n\leq Bm_0/\gamma$, (\ref{lem-upperbound-2}) holds because its right side is bigger than $1$. Substituting (\ref{lem-upperbound-2}) into (\ref{lem-upperbound-1}), we get that
	Clearly, $\frac{C_2}{\sqrt{n}} >1$ when $n\leq  \lfloor m_0/\gamma\rfloor \leq B  \lfloor m_0/\gamma\rfloor$. Therefore, we get that
	\begin{align*}
		\Pb\left[ L_{P^{\pi_e}}(\hat{C}_{\tilde{\tau}})\leq\epsilon\right]<1-\delta+\frac{C_1B/\sqrt{\gamma}+C_2 }{\sqrt{n}}  %+e^{\frac{4m_0}{B\gamma}}e^{-\frac{2n}{B^2}},%\label{lem-upperbound-3}
	\end{align*}
	for all  $n\geq 1$. 
\end{proof}

\subsection{Proof of Theorem \ref{thm-bothsides}}\label{pf-thm-bothsides-2}
Essentially, Theorem \ref{thm-bothsides} can be derived by the argument similar to Theorem \ref{lem-upperbound}. The details are stated below.

\begin{proof}
	Denote $ \epsilon' \coloneq \epsilon-\Delta_\epsilon$ and define $\delta'\coloneq F_{\Bin(m,\epsilon')}(k(m,\epsilon,\delta))$. We have
	$
	k(m,\epsilon',\delta')=k(m,\epsilon,\delta)
	$
	by the definition (\ref{def-k}). Then, for all $m\geq m_0$ such that 
	$
	k(m,\epsilon',\delta')=k(m,\epsilon,\delta) \neq -1
	$, it follows from (\ref{eq-pf-lem-1}) that 
	$$
	\Pb\left[  L_{P^{\pi_e}}(\hat{C}_{\tilde{\tau}}) \leq  \epsilon' \mid N_{\mathrm{rs}}=n_{\mathrm{rs}}, \Ds^{\mathrm{rs}}_1\right] < 1-\delta'+\frac{C_3}{\sqrt{m}},
	$$
	with $C_3=7/\sqrt{\epsilon'(1-\epsilon')}$. Therefore, 
	\begin{align*}
		&\Pb\left[\epsilon' < L_{P^{\pi_e}}(\hat{C}_{\tilde{\tau}})\leq\epsilon \mid N_{\mathrm{rs}}=n_{\mathrm{rs}}, \Ds^{\mathrm{rs}}_1 \right]\\ =&1-\Pb\left[L_{P^{\pi_e}}(\hat{C}_{\tilde{\tau}})\leq\epsilon' \mid N_{\mathrm{rs}}=n_{\mathrm{rs}}, \Ds^{\mathrm{rs}}_1 \right]
		-\Pb\left[L_{P^{\pi_e}}(\hat{C}_{\tilde{\tau}})>\epsilon \mid N_{\mathrm{rs}}=n_{\mathrm{rs}}, \Ds^{\mathrm{rs}}_1 \right]\\
		>&\delta'-\delta-\frac{C_3}{\sqrt{m}}.
	\end{align*}
	We next show that $\delta'$ approaches 1 at an exponential rate. Hoeffding's inequality yields that for any $t>0$,
	$
	F_{\text{Bin}(m,\epsilon)}(m\epsilon-mt)\leq \exp\{-2mt^2\}.
	$
	Choosing  $t=\sqrt{\frac{\log \delta}{-2m}}$ leads to a lower bound for $k(m,\epsilon,\delta)$:
	\begin{equation}\label{k-bound}
		k(m,\epsilon,\delta)\geq m(\epsilon- \sqrt{ \frac{\log\delta}{-2m}} ).
	\end{equation}
	When $\sqrt{ \frac{\log\delta}{-2m}} < \Delta_\epsilon$, it follows again from Hoeffding's inequality that
	$$
	\delta'\geq F_{\Bin(m,\epsilon')}[ m(\epsilon- \sqrt{ \frac{\log\delta}{-2m}} ) ] \geq 1- \exp( -2m[\Delta_\epsilon- \sqrt{ \frac{\log\delta}{-2m}} ]^2) .
	$$
	Therefore, for $m> m_1\coloneq m_0\vee \frac{\log\delta}{-2\Delta_\epsilon^2}$, we have
	$$
	\Pb\left[\epsilon' < L_{P^{\pi_e}}(\hat{C}_{\tilde{\tau}})\leq\epsilon \mid N_{\mathrm{rs}}=n_{\mathrm{rs}}, \Ds^{\mathrm{rs}}_1 \right] > 1-\delta -\frac{C_4}{\sqrt{m}},
	$$
	where $C_4=C_3+\frac{\sqrt{-2\log\delta}+1}{2\Delta_\epsilon}$ is derived from straightforward calculations. After marginalizing, we have for $n_{\mathrm{rs}} \geq \lfloor m_1/\gamma \rfloor +1$,
	$$
	\Pb\left[\epsilon' < L_{P^{\pi_e}}(\hat{C}_{\tilde{\tau}})\leq\epsilon \mid N_{\mathrm{rs}}=n_{\mathrm{rs}} \right] > 1-\delta -\frac{C_4}{\sqrt{\gamma n_{\mathrm{rs}}}}.
	$$
	Finally, similarly to the proof of Theorem \ref{lem-upperbound}, we obtain for $n\geq \lfloor m_1/\gamma\rfloor +1 $ that
	\begin{align*}
		\Pb\left[\epsilon' < L_{P^{\pi_e}}(\hat{C}_{\tilde{\tau}})\leq\epsilon \right] &> \sum_{n_{\mathrm{rs}}=\lfloor m_1/\gamma\rfloor +1}^{n}\Pb(N_{\mathrm{rs}}=n_{\mathrm{rs}})(1-\delta -\frac{C_4}{\sqrt{\gamma n_{\mathrm{rs}}}})\\
		&\geq (1-\delta)(1-F_{\Bin(n,1/B)}(\lfloor m_1/\gamma \rfloor )) - \frac{C_4B}{\sqrt{\gamma n}}\\
		& \geq 1-\delta - \frac{(1-\delta)C_5 + C_4B/\sqrt{\gamma}}{\sqrt{n}},
	\end{align*}
	where $C_5 = \sqrt{B \lfloor m_1/\gamma \rfloor} + B/2$. As it also holds for $n\leq \lfloor m_1/\gamma\rfloor$, we complete the proof.
\end{proof}

\subsection{Proof of Theorem \ref{thm-asymptotic}}\label{Proof of Theorem thm-asymptotic}
This proof largely follows the argument of Theorem 1 in \cite{sesia2020comparison}, although the threshold $\tilde{\tau}$ employed here differs from the $1-\epsilon$ empirical quantile used therein. The following consistency assumption, which is analogous to Assumption A4 in \cite{lei2018distribution}, is weaker than requiring $L_2$-convergences and can be satisfied for certain algorithms under appropriate conditions, such as random forests \citep{meinshausen2006quantile}.

\begin{ass}\label{ass-consistency}
	Denote by $l$ the size of the training set $\Ds^{\mathrm{rs}}_1$ used to fit the conditional quantile functions $\hat{q}_{\epsilon_{\mathrm{lo}}}$ and $\hat{q}_{\epsilon_{\mathrm{up}}}$. For sufficiently large $l$, the following conditions hold:
	\begin{align*}
		\Pb\left[ \Eb \left[ \big(\hat{q}_{\epsilon_{\mathrm{lo}}}(S_{n+1}) -q_{\epsilon_{\mathrm{lo}}}(S_{n+1}) \big)^2 \mid \Ds^{\mathrm{rs}}_1 \right] \leq \eta_l/2 \right] &\geq 1- \rho_l/2, \\
		\Pb\left[ \Eb \left[ \big(\hat{q}_{\epsilon_{\mathrm{up}}}(S_{n+1}) -q_{\epsilon_{\mathrm{up}}}(S_{n+1}) \big)^2 \mid \Ds^{\mathrm{rs}}_1 \right] \leq \eta_l/2 \right] &\geq 1- \rho_l/2,
	\end{align*}
	for some sequences $\eta_l = o(1)$ and $\rho_l = o(1)$, as $l\to\infty$. 
\end{ass}
In addition, a regularity assumption is needed.
\begin{ass}\label{ass-regularity}
	For $(S,R)\sim P^{\pi_e}$, the probability density of the random variable
	$$
	T\coloneq\max\{q_{\epsilon_{\mathrm{lo}}}(S) - R, R - q_{\epsilon_{\mathrm{up}}}(S)\}
	$$
	is bounded away from zero in a neighborhood of zero.
\end{ass}

\begin{proof}
	It suffices to show that, as $|\Ds_{\mathrm{rs}}|=n^{\mathrm{rs}}\to\infty$,
	\begin{enumerate}[label=(\roman*)]
		\item $|\hat{q}_{\epsilon_{\mathrm{lo}}}(S_{n+1})-q_{\epsilon_{\mathrm{lo}}}(S_{n+1})|=o_\Pb(1)$ and $|\hat{q}_{\epsilon_{\mathrm{up}}}(S_{n+1})-q_{\epsilon_{\mathrm{up}}}(S_{n+1})|=o_\Pb(1)$;
		\item $\tilde{\tau}=o_\Pb(1)$.
	\end{enumerate}
	Then the proof will be completed by the triangle inequality.
	
	(i) Define the random set
	$$
	B_{l,\mathrm{lo}}=\{s:|\hat{q}_{\epsilon_{\mathrm{lo}}}(s)-q_{\epsilon_{\mathrm{lo}}}(s) | \geq \eta_l^{1/3}\}, B_{l,\mathrm{up}}=\{s:|\hat{q}_{\epsilon_{\mathrm{up}}}(s)-q_{\epsilon_{\mathrm{up}}}(s) | \geq \eta_l^{1/3}\},
	$$
	and $B_l=B_{l,\mathrm{lo}} \cup B_{l,\mathrm{up}}$. We have by Markov's inequality and Assumption \ref{ass-consistency} that
	\begin{align*}
		&\Pb(S_{n+1} \in B_l\mid \Ds^{\mathrm{rs}}_1)  \leq \Pb(S_{n+1}\in B_{l,\mathrm{lo}}\mid \Ds^{\mathrm{rs}}_1) + \Pb(S_{n+1}\in B_{l,\mathrm{lo}}\mid \Ds^{\mathrm{rs}}_1) \\
		= &\Pb( |\hat{q}_{\epsilon_{\mathrm{lo}}}(S_{n+1})-q_{\epsilon_{\mathrm{lo}}}(S_{n+1}) | \geq \eta_l^{1/3} \mid \Ds^{\mathrm{rs}}_1) + \Pb(|\hat{q}_{\epsilon_{\mathrm{up}}}(S_{n+1})-q_{\epsilon_{\mathrm{up}}}(S_{n+1}) | \geq \eta_l^{1/3} \mid \Ds^{\mathrm{rs}}_1)\\
		\leq & \eta_l^{-2/3}\Eb\left[ \big(\hat{q}_{\epsilon_{\mathrm{lo}}}(S_{n+1}) -q_{\epsilon_{\mathrm{lo}}}(S_{n+1}) \big)^2  \mid \Ds^{\mathrm{rs}}_1 \right] +  \eta_l^{-2/3}\Eb\left[ \big(\hat{q}_{\epsilon_{\mathrm{up}}}(S_{n+1}) -q_{\epsilon_{\mathrm{up}}}(S_{n+1}) \big)^2 \mid \Ds^{\mathrm{rs}}_1  \right]\\
		\leq  &\eta_l^{1/3},
	\end{align*}
	with probability at least $1-\rho_l$. This implies $|\hat{q}_{\epsilon_{\mathrm{lo}}}(S_{n+1})-q_{\epsilon_{\mathrm{lo}}}(S_{n+1})|=o_\Pb(1)$ and $|\hat{q}_{\epsilon_{\mathrm{up}}}(S_{n+1})-q_{\epsilon_{\mathrm{up}}}(S_{n+1})|=o_\Pb(1)$.

	(ii) Consider the following partition of the calibration set $\Ds^{\mathrm{rs}}_2$:
	$$
	\Ds^{\mathrm{rs}}_{2,a} = \{ (S_i,R_i)\in\Ds^{\mathrm{rs}}_2 : S_i \in B_l^c \} ,\; \Ds^{\mathrm{rs}}_{2,b} = \{ (S_i,R_i)\in\Ds^{\mathrm{rs}}_2 : S_i \in B_l \}.
	$$
	Since $B_l$ only depends on $\Ds^{\mathrm{rs}}_1$, the size of $\Ds^{\mathrm{rs}}_{2,b}$ conditional on $\Ds^{\mathrm{rs}}_1$ can be bounded using Hoeffding's inequality as
	$$
	\Pb(|\Ds^{\mathrm{rs}}_{2,b}| \geq m\eta_l^{1/3} + t) \leq \Pb( \sum_{(S_i,R_i)\in \Ds^{\mathrm{rs}}_2}  \mathbbm{1}(\{ S_i\in B_l \}) \geq m\Pb(S_i\in B_l) +t) \leq \exp\left( -\frac{2t^2}{m} \right).
	$$
	Choosing $t=C\sqrt{m\log m}$ leads to $|\Ds^{\mathrm{rs}}_{2,b}|=o_\Pb(m)=o_\Pb(n^{\mathrm{rs}})$.
	
	Now, for each $(S_i,R_i)\in\Ds^{\mathrm{rs}}_2 $, define $T_i=\max\{ q_{\epsilon_{\mathrm{lo}}}(S_i) - R_i, R_i - q_{\epsilon_{\mathrm{up}}}(S_i)\}$. By the definition (\ref{def-tilde-tau-i}) of $\tau_i$, 
	it can be easily derived that
	\begin{equation}\label{pf-thm-asymptotic-eq-1}
		|T_i-\tau_i|\leq \eta_l^{1/3}, \text{ for } i \text{ s.t. } (S_i,R_i)\in\Ds^{\mathrm{rs}}_{2,a}.
	\end{equation}
	Recall that $\tilde{\tau}=\tau_{(m-k(m,\epsilon,\delta))}$, we also define the $k$-th smallest $T_i$ as $T_{(k)}$. In addition, when restricted to the data set $\Ds^{\mathrm{rs}}_{2,a}$, define $\tau_{(k)}^a$ and ${T}_{(k)}^a$ as the $k$-th smallest $\tau_i$ and $T_i$ for $i$ s.t. $(S_i,R_i)\in\Ds^{\mathrm{rs}}_{2,a}.$ Moreover, as demonstrated in $(\ref{k-bound})$, a similar upper bound for $k(m,\epsilon,\delta)$ can be established as
	\begin{equation}\label{pf-thm-asymptotic-eq-2}
		m(\epsilon- \sqrt{ \frac{\log\delta}{-2m}} )\leq k(m,\epsilon,\delta) \leq m(\epsilon+\sqrt{\frac{\log(1-\delta)}{-2m}}).
	\end{equation}
	For $n^{\mathrm{rs}}$ large enough, without loss of generality, we assume $|\Ds^{\mathrm{rs}}_{2,b}|<m-k(m,\epsilon,\delta)$ because $|\Ds^{\mathrm{rs}}_{2,b}|=o_\Pb(m)$. Then, it can be straightforward verified that
	$$
	\tau_{(m-k(m,\epsilon,\delta) - |\Ds^{\mathrm{rs}}_{2,b}|)}^a \leq \tau_{(m-k(m,\epsilon,\delta))}\leq \tau_{(m-k(m,\epsilon,\delta))}^a.
	$$
	Together with (\ref{pf-thm-asymptotic-eq-1}), we have
	$$
	T_{(m-k(m,\epsilon,\delta) - |\Ds^{\mathrm{rs}}_{2,b}|)}^a -\eta_l^{1/3} \leq \tilde{\tau} \leq T_{(m-k(m,\epsilon,\delta))}^a+\eta_l^{1/3},
	$$
	which in turn yields
	$$
	T_{(m-k(m,\epsilon,\delta) - |\Ds^{\mathrm{rs}}_{2,b}|)} -\eta_l^{1/3} \leq \tilde{\tau} \leq T_{(m-k(m,\epsilon,\delta) +  |\Ds^{\mathrm{rs}}_{2,b}| )}  +\eta_l^{1/3}.
	$$
	Therefore, it suffices to prove $T_{(m-k(m,\epsilon,\delta) -  |\Ds^{\mathrm{rs}}_{2,b}| )}$ and $T_{(m-k(m,\epsilon,\delta) +  |\Ds^{\mathrm{rs}}_{2,b}| )}$ is $o_\Pb(1)$.
	
	Actually, for any sufficiently small $\alpha>0$, we have $\epsilon_\alpha=:\Pb(T_i>\alpha) < \epsilon $ for each $T_i$ by Assumption \ref{ass-regularity} and the fact that $\Pb(T_i> 0)=\epsilon$. Then, Hoeffding's inequality and (\ref{pf-thm-asymptotic-eq-2}) yields for sufficiently large $m$ that
	\begin{align*}
		\Pb( T_{(m-k(m,\epsilon,\delta) +  |\Ds^{\mathrm{rs}}_{2,b}| )} \leq \alpha ) &= F_{\Bin(m,\epsilon_\alpha)} (k(m,\epsilon,\delta) -  |\Ds^{\mathrm{rs}}_{2,b}| )\\&\geq 1- \exp\left( -2m[\frac{k(m,\epsilon,\delta) -  |\Ds^{\mathrm{rs}}_{2,b}| }{m} -  \epsilon_\alpha]^2\right).
	\end{align*}
	One can prove in a similar manner that $\Pb( T_{(m-k(m,\epsilon,\delta) +  |\Ds^{\mathrm{rs}}_{2,b}| )} \geq -\alpha )\to 1$ as $m\to \infty$. Thus, we obtain $ T_{(m-k(m,\epsilon,\delta) +  |\Ds^{\mathrm{rs}}_{2,b}|)} = o_\Pb(1)$, and $ T_{(m-k(m,\epsilon,\delta) - |\Ds^{\mathrm{rs}}_{2,b}|)} = o_\Pb(1)$ analogously. 
\end{proof}

\subsection{Proof of Theorem \ref{thm-est-pib}}\label{pf-thm-est-pib}

\begin{proof}
	Conditional on a realization of $\Ds_1$, it can be shown, analogous to Proposition \ref{pro-1}, that each sample in $\Ds^{\mathrm{rs}}_2$ independently follows the joint distribution
	\begin{equation*}
		\hat{P}( s, r)=\hat{P}( s, r;\Ds_1)\coloneq P_S(s) \int_\As\frac{\hat{w}(s,a)}{\Eb_{\Ds_1}[\hat{w}(S,A)]} \pi_b(a|s)P_{R}(r|s,a) \drm a.
	\end{equation*}
	By assuming the test data $(S_{n+1},R_{n+1})\sim\hat{P}$, we can obtain from Theorem \ref{thm-1} that
	\begin{equation}\label{pf-estpib-eq1}
		\Pb[L_{\hat{P}}(\hat{C}_{\tilde{\tau}})\leq\epsilon\mid \Ds_1]\geq 1-\delta.
	\end{equation}
	Denote by $d_{\text{TV}}(\hat{P},P^{\pi_e})$ the total variation distance between $\hat{P}$ and $P^{\pi_e}$. We have
	\begin{align*}
		d_{\text{TV}}(\hat{P},P^{\pi_e})&=\frac{1}{2}\int_\Ss\int_\Rb\left|\hat{P}(s,r)-P^{\pi_e}( s, r) \right|\drm r \drm s \\
		&\leq  \frac12\int_\Ss\int_\Rb\int_\As\left| \frac{\hat{w}(s,a)}{\Eb_{\Ds_1}[\hat{w}(S,A) ]} - w(s,a)\right| P_{R}(r|s,a) \pi_b(a|s)P_S(s)\drm a\drm r \drm s  \\
		&= \frac12\Eb_{\Ds_1}  \left| \frac{\hat{w}(S,A)}{\Eb_{\Ds_1}[\hat{w}(S,A) ]} - w(S,A)\right|\\
		&\leq \frac12  \Eb_{\Ds_1}  \left| \hat{w}(S,A) - w(S,A)\right| + \frac12 \Eb_{\Ds_1}  \left| \frac{\hat{w}(S,A)}{\Eb_{\Ds_1}[\hat{w}(S,A) ]} - \hat{w}(S,A)\right| \\
		&\leq \frac12  \Eb_{\Ds_1}  \left| \hat{w}(S,A) - w(S,A)\right| + \frac12   \left| \Eb_{\Ds_1}[\hat{w}(S,A)]-1\right| \\
		&\leq  \Eb_{\Ds_1}  \left| \hat{w}(S,A) - w(S,A)\right| = \Delta_w,
	\end{align*}
	as $\Eb_{\Ds_1}[w(S,A) ]=1$. By (\ref{pf-estpib-eq1}) and the fact  $\left|L_{\hat{P}}(\hat{C}_{\tilde{\tau}})-L_{P^{\pi_e}}(\hat{C}_{\tilde{\tau}})\right| \leq  d_{\text{TV}}(\hat{P},P^{\pi_e}),$ we have
	$$
	\Pb\left [L_{P^{\pi_e}}(\hat{C}_{\tilde{\tau}})\leq\epsilon+\Delta_w \right]
	\geq \Pb\left [L_{\hat{P}}(\hat{C}_{\tilde{\tau}})\leq\epsilon \right]
	\geq 1-\delta.
	$$
	Thus, (\ref{thm-est-pib-eq1}) is obtained. (\ref{thm-est-pib-eq2}) follows similarly and is omitted here. 
\end{proof}

\subsection{Proof of Theorem \ref{thm-MLE}}\label{pf-thm-MLE}

\begin{proof}
	By Jensen's inequality and the fact that $\|p-q\|_{L_1}=2d_{\text{TV}}(p,q)$ for two distributions $p$ and $q$, we have
	\begin{align*}
		\Eb_{ \Ds_1} \left| \hat{w}(S,A)-w(S,A)\right|&=\int_\Ss\int_\As\left| \frac{\pi_e}{\hat{\pi}_b}( a|s)-\frac{\pi_e}{\pi_b}( a|s) \right|\pi_b(a|s)P_S( s)\drm a \drm s \\
		&=\int_\Ss\int_\As \frac{\pi_e}{\hat{\pi}_b}( a|s) \left|\pi_b( a|s) -\hat{\pi}_b( a|s) \right|P_S( s)\drm a \drm s \\
		&\leq B_{\Pi}\int_\Ss\int_\As  \left|\pi_b( a|s) -\hat{\pi}_b( a|s) \right|P_S( s) \drm a \drm s\\
		&\leq B_{\Pi}\left( \int_\Ss\left| \int_\As  \left|\pi_b(a|s) -\hat{\pi}_b(a|s) \right|\drm a \right|^2 P_S( s) \drm s \right)^{1/2}  \\
		&= 2B_{\Pi}\left( \int_\Ss [d_{\text{TV}}(\pi_b(\,\cdot\mid s),\hat{\pi}_b(\,\cdot\mid s))]^2  P_S(s)\drm s \right)^{1/2}.
	\end{align*}
	The proof then follows from (\ref{thm-est-pib-eq1}) and Theorem 21 in \citep{agarwal2020advances}, which shows that, with probability at least $1-\delta_0$,
	$$
	\Delta_w=\Eb_{ \Ds_1 } \left| \hat{w}(S,A)-w(S,A)\right|\leq 2B_{\Pi}\sqrt{\frac{2\log(|\Pi|/\delta_0)}{|\Ds_1|}}.
	$$
\end{proof}

\end{document}